\documentclass[twoside,11pt]{article}
\usepackage[]{algorithm2e}
\usepackage{algpseudocode}
\usepackage{amssymb}
\usepackage{amsthm}
\usepackage{graphicx}
\usepackage{hyperref}
\usepackage{lscape}
\usepackage{mathrsfs}
\usepackage{mathtools}
\usepackage{multirow}
\usepackage{setspace}
\usepackage{subfigure}
\usepackage{tikz}
\usepackage{wrapfig}

\newcommand{\bss}{\boldsymbol}
\newcommand{\bs}{\mathbf}

\newtheorem{Criterion}{Criterion}
\newtheorem{Definition}{Definition}
\newtheorem{Lemma}{Lemma}
\newtheorem{Remark}{Remark}
\newtheorem{Theorem}{Theorem}

\begin{document}

\title{Learning Hierarchical Feature Space Using CLAss-specific Subspace Multiple Kernel - Metric Learning for Classification}

\author{Yinan Yu, Tomas McKelvey}

\maketitle

\begin{abstract}
Metric learning for classification has been intensively studied over the last decade. The idea is to learn a metric space induced from a normed vector space on which data from different classes are well separated. Different measures of the separation thus lead to various designs of the objective function in the metric learning model. One classical metric is the Mahalanobis distance, where a linear transformation matrix is designed and applied on the original dataset to obtain a new subspace equipped with the Euclidean norm. The kernelized version has also been developed, followed by Multiple-Kernel learning models.
In this paper, we consider metric learning to be the identification of the best kernel function with respect to a high class separability in the corresponding metric space. The contribution is twofold:
1) No pairwise computations are required as in most metric learning techniques;
2) Better flexibility and lower computational complexity is achieved using the CLAss-Specific (Multiple) Kernel - Metric Learning (CLAS(M)K-ML).
The proposed techniques can be considered as a preprocessing step to any kernel method or kernel approximation technique.
An extension to a hierarchical learning structure is also proposed to further improve the classification performance, where on each layer, the CLASMK is computed based on a selected ``marginal'' subset and feature vectors are constructed by concatenating the features from all previous layers.
\end{abstract}
\begin{center}
Classification, metric learning, Multiple-Kernel, subspace models
\end{center}

\section{Introduction}
\label{sec:intro}
Given a nonempty input set $\mathcal{X}\times \mathcal{Y}$, a training set
$\mathcal{D}=\{(\bs{x}_i,y_i) : \bs{x}_i\in \mathcal{X}, y_i\in \mathcal{Y},~\forall i\in \{1\cdots N\}\}$ is defined as a sample drawn from the random distribution of $\mathcal{X}\times \mathcal{Y}$.
For a classification problem, one learns a function $f:\mathcal{X}\rightarrow \mathbb{R}$, such that for any $\bs{x}\in \mathcal{X}$, $f(\bs{x})$ can predict the associated label $y$ as accurate as possible.
One of the simplest classifiers is the linear classifier. For example, let $\bs{x}\in\mathbb{R}^p$ and $y\in\{0,1\}$, a linear function $f(\bs{x})=\bs{w}^T\bs{x}+b$
parameterized by $\bs{w}\in\mathbb{R}^p$ and $b\in\mathbb{R}$ makes the prediction $\hat{y}=1$, if $f(\bs{x})>0$ and vice versa. However, despite having the advantage of computational simplicity, linear models have limited capacity of data modeling. Nonlinear methods, on the other hand, offer more flexibility.

One important family of nonlinear techniques is the kernel method, where a high dimensional feature space $\mathcal{F}$ is constructed by a nonlinear mapping $\varphi:\mathcal{X}\rightarrow \mathcal{F}$ associated with a kernel function $k(\cdot,\bs{x})\triangleq \varphi(\bs{x})$, $\forall \bs{x}\in\mathcal{X}$. A kernel function can be perceived as a positive-definite \cite{Scholkopf_nonlinear_component_1998, Kung_kernel_methods_2014} real-valued function that takes any $\bs{x}, \tilde{\bs{x}}\in \mathcal{X}$ as its input, such that $k(\bs{x}, \tilde{\bs{x}}) \triangleq \varphi(\bs{x})^T\varphi(\tilde{\bs{x}})$.
Due to the possibly high dimensionality of $\mathcal{F}$, explicit computations in $\mathcal{F}$ is typically prohibitive. Fortunately, if the problem can be formulated in a way that only computations of inner products are conducted in $\mathcal{F}$, all such explicit computations can be replaced by the evaluations of the kernel function on the input set $\mathcal{X}$. This reduction of computational complexity is often referred to as the ``kernel trick''.
The most classical example is the Support Vector Machines (SVM) \cite{Vapnik_the_nature_1995}.

Moreover, given a kernel function $k$, let $q = {dim(\mathcal{F})}$ denote the dimensionality of the feature space $\mathcal{F}$ induced by $k$. The classifier can be parameterized by
\begin{equation}
\label{eqa:solutionw}
f(\bs{x})= \bs{w}^T\varphi(\bs{x})+b
\end{equation}
where $\bs{w}\in\mathbb{R}^{q}$ and $b\in{\mathbb{R}}$. The classification rule is then:
\begin{equation}
\label{eqa:multiclass_t}
\hat{y} = \begin{cases}+1,& \text{if } f(\bs{x})\geq 0\\ -1, & \text{otherwise}\end{cases}
\end{equation}
\footnote{Similarly, for the multiclass scenario, given $C$ classes, we have
\begin{equation}
\label{eqa:multiclass_xi}
\bss{\xi}\triangleq \bs{W}^T\varphi(\bs{x})+\bs{b}
\end{equation}
where $\bs{W}\in\mathbb{R}^{q\times C}$, $\bs{b}\in\mathbb{R}^C$ and the label is then estimated using $\hat{y}=\underset{c\in\{1\cdots C\}}{\max} ~\bss{\xi}(c)$ with $\bss{\xi}(c)$ denoting the $c^{th}$ element of vector $\bss{\xi}$. }
The representer theorem \cite{Scholkopf_nonlinear_component_1998} states that the optimal $\bs{w}$ associated with a given training set can be written as a linear combination of all data vectors from this set. That is, given a matrix $\bss{\Phi}=\begin{bmatrix}\varphi(\bs{x}_i) \end{bmatrix}_{\forall (\bs{x}_i,y_i)\in \mathcal{D}}$, there exists a coefficient vector $\bs{a}$, such that $\bs{w}=\bss{\Phi}\bs{a}$. In other words, $\bs{w}$ lies in the column space of $\bss{\Phi}$.
Due to this subspace property, to solve for $\bs{w}$ with a training set of size $N$, it typically requires computations with complexity of $\mathcal{O}(N^3)$ and storage capacity of order $\mathcal{O}(N^2)$. With a growing sample size $N$, the amount of computations needed quickly becomes prohibitive.

One way to reduce the complexity is to approximate the range $\mathcal{R}(\bss{\Phi})$ with fewer vectors since the representer theorem states that $\bs{w}\in \mathcal{R}(\bss{\Phi})$.
More precisely, one would like to find a matrix $\bs{U}\subset \mathbb{R}^{q\times m}$ containing a set of orthonormal vectors
such that $\mathcal{R}(\bs{U}) \approx \mathcal{R}(\bss{\Phi})$ with $m\ll N$.

This is the essential objective for kernel approximation techniques.
In other words, for such methods to make sense, the underlying data generating assumption could be characterized by the following subspace model:
\begin{eqnarray}
\label{eqa:model}
\bss{\varphi}= \bs{U}  \bss{\beta} + \bs{e}
\end{eqnarray}
where $\bss{\beta}$ contains the coordinates in the subspace spanned by the columns of $\bs{U}$ and $\bs{e}$ is a random vector with finite covariance and zero mean. Furthermore, we assume that $\bs{U}\perp\bs{e}$.
Hence, the solution vector $\bs{w}$ in Eq.~\eqref{eqa:solutionw} can be written as
\begin{eqnarray}
\nonumber
\bs{w}&=&\left(\bs{U}\begin{bmatrix}\bss{\beta}_1\cdots\bss{\beta}_n \end{bmatrix}+\begin{bmatrix}\bs{e}_1\cdots\bs{e}_n \end{bmatrix}\right)\bs{a}\\
&=& \bs{U}\bss{\alpha} + \bs{v}
\end{eqnarray}
where $\bss{\alpha}=\begin{bmatrix}\bss{\beta}_1\cdots\bss{\beta}_n \end{bmatrix}\bs{a}$ and $\bs{v}=\begin{bmatrix}\bs{e}_1\cdots\bs{e}_n \end{bmatrix}\bs{a}$.

From a kernel approximation point of view, we attempt to approximate the positive-semidefinite (PSD) kernel matrix $\bs{K}=\bss{\Phi}^T\bss{\Phi}$ using a matrix
\begin{equation}
\label{eqa:kernel_approximation}
\bs{L}=\bss{\Phi}^T\bs{U}\bs{U}^T\bss{\Phi},
\end{equation}
such that some prescribed optimality is achieved. For example, one classical criterion is the Frobenius norm of the differences $\|\bs{K}-\bs{L}\|_F$.
Such criteria are designed to approximate the kernel matrix $\bs{K}$ with a low rank matrix $\bs{L}$ so that the computational complexity can be reduced from $\mathcal{O}(N^3)$ to $\mathcal{O}(m^3)$, where $m$ is the rank of the approximated kernel matrix. Algorithms and corresponding analysis can be found in \cite{Bach_kernel_independent_2002, Bach_predictive_low_2005, Cortes_on_the_2010, Kumar_sampling_methods_2012, Wang_improving_cur_2013}.

However, in a classification problem, the quality of the kernel matrix approximation evaluated by $\|\bs{K}-\bs{L}\|$ is not the most interesting by itself. A more desirable attribute is the class separability, which can be improved by metric learning techniques.

In metric learning, one defines a parameterized metric space as the feature space. For example, one such parameterization is the Mahalanobis distance metric defined as
\begin{equation}
d^2(\bs{x},\bs{\tilde{x}})= \|\bs{A}^T\left(\bs{x} - \bs{\tilde{x}}\right)\|_2^2
\end{equation}
with the design matrix $\bs{A}$, such that $d^2(\bs{x},\bs{\tilde{x}})$ is large when $\bs{x}$, $\bs{\tilde{x}}$ are from different classes, and small otherwise.
Of course, despite its computational simplicity, linear model does not always provide the most efficient solutions and nonlinear parameterizations exist in the literature.
In particular, kernel formulations in combination with metric learning have been studied by many authors over the past decade \cite{Xing_distance_metric_2002, Jain_metric_and_2012}. These methods require the kernel parameters to be chosen in advance as hyperparameters using, for example, grid search and cross-validation. There also exist some gradient based techniques designed for finding the kernel parameters \cite{Chappelle_choosing_multiple_2002}.

A more complex model is the Multiple-Kernel (MK) model \cite{Gonen_multiple_kernel_2011}, where the kernel function is written as a combination of several predefined kernel functions. As summarized in \cite{Gonen_multiple_kernel_2011}, MK models mainly utilize target functions that fall into three groups: 1) similarity measure, 2) structural risk functions and 3) Bayesian functions. For the computational complexity, the training processes have been grouped into two categories: one-step methods and two-step methods, where two-step methods alternate the learning process between the kernel combination and the parameters of the base learner (such as SVM), while one-step methods learn both step in one run. Amongst one-step methods, some use fixed rules and others use optimization approaches. Fixed rules are fast to train, whereas optimization formulations take longer time in general. In this paper, we propose two one-step methods: CLAss-specific Subspace Kernel Metric Learning (CLASK-ML) and CLAss-Specific Multiple-Kernel Metric Learning (CLASMK-ML). Both methods are based on a similarity measure. Moreover, CLASK-ML can be implemented using fixed rules, whereas CLASMK-ML requires optimizations.
We parameterize the distance metric only by the kernel function without associating it with a linear transformation matrix $\bs{A}$ or any specific type of classifiers. The measure of the quality for the metric space is the probability of the within-class distance exceeding the size of the between-class distance.
This can be a preprocessing step to any learning model.
Furthermore, the learning process does not require any pairwise computations as in most metric learning techniques, which results in reduced computational complexity.
In addition, the block-based structure provided by the class-specific learning model also reduces the computational time for the kernel matrix.

This paper is organized as follows. In Sec.~\ref{sec:formulation}, an evaluation of the class separability along with a parametric subspace model is presented. The class separability is evaluated based on the probability of the within-class distance exceeding the between-class distance. Upper bounds are derived with respect to this evaluation. One criterion for choosing an appropriate kernel function is proposed accordingly. The idea is then extended to a new class-specific subspace learning model equipped with Multiple-Kernel for better flexibility. Algorithms and implementations are presented in Sec.~\ref{sec:algorithm}. The hierarchical feature transformation is presented in Sec.~\ref{sec:hierarchy}. Experimental results are shown in Sec.~\ref{sec:results}.

\section{Problem formulation}
\label{sec:formulation}
\subsection{Distance metric and subspace model for a given kernel function}
Given a training set $\mathcal{D}$, the goal is to find a metric space $\mathcal{F}$ induced by a predefined kernel function $k$, such that a good separation between different classes is achieved in $\mathcal{F}$. This metric space is then used as the feature space for the given classification problem.

In order to construct a feature space with satisfying class separability, the objective and its parameterization need to be quantified.
\paragraph{Distance metric and the objective} Given a classification problem with $C$ classes and a nonlinear mapping of data vectors denoted by $\varphi(\bs{x})$. Let $\bss{\varphi}_c=\varphi(\bs{x})$ for $\bs{x}\in$ class $c$ and $\bss{\nu}_{c'}=\varphi(\bs{x})$ for $\bs{x}\in$ class $c'$. The Euclidean distance between data from class $c$ and $c'$ after applying the nonlinear transformation is defined as:
\begin{eqnarray}
\label{eqa:distance}
D_{c,c'}:=d^2(\bss{\varphi}_c,\bss{\nu}_{c'})= \|\bss{\varphi}_c-\bss{\nu}_{c'}\|_2^2
\end{eqnarray}
This distance is usually called the ``between-class'' distance when $c\neq c'$ and the ``within-class'' distance when $c= c'$. Note that for simplicity, we denote $\tilde{c}\in\{1\cdots C\}\setminus c$ throughout the paper.

\paragraph{Statistical modeling} In this paper, we model the problem using a statistical framework. That is, we assume that $\bs{x}$ is a random vector drawn from an unknown multivariate probability distribution and hence $D_{c,c'}$ is considered a random variable.
The class separability can be measured by the relation between the within-class distance $D_{c,c}$ and the between-class distance $D_{c,\tilde{c}}$. One possibility is the probability $\mathbb{P}(D_{c,c}  > \mathbb{E}(D_{c,\tilde{c}}))$, which is small for a good separability.
Hence, the goal is to find a feature space with a low $\mathbb{P}(D_{c,c}> \mathbb{E}(D_{c,\tilde{c}}))$.
\paragraph{Parameterization} For conducting further analysis, we employ a parametric subspace model.
Given a feature space $\mathcal{F}$ induced by a predefined kernel function $k(\cdot,\bs{x}_c):\bs{x}_c\mapsto \bss{\varphi}_c$, the underlying assumption is the following:
\begin{equation}
\label{eqa:model_decomp}
\bss{\varphi}_c = \bs{U}_{c}\bss{\beta}_{c}+\bs{e}_c
\end{equation}
where $\bs{U}_{c}$ contains a set of orthonormal vectors that span the subspace generating random vectors from class $c$; $\bss{\beta}_{c}$ and $\bs{e}_c$ contains the random coefficients and the zero-mean random noise vector, respectively.
Furthermore, we make the following assumptions for all class $c$: 1) We always use normalized kernel function, i.e. $\|\bss{\varphi}_c\|_2^2=1$; 2) The noise vector $\bs{e}_c$ is zero-mean with $\mathbb{E}(\|\bs{e}_c\|_2^2)=\sigma_e^2$ and 3) Each dimension of vectors $\bs{U}_{c}\bss{\beta}_{c}$, $\bs{e}_{c'}$ and $\bs{e}_{c''}$ are uncorrelated random variables $\forall c, c', c''\in \{1 \cdots C\}$.
\paragraph{Analysis} Given the basis matrix $\bs{U}_c$ for each class $c$, we conduct the analysis by deriving an upper bound on the quantity of interest $\mathbb{P}(D_{c,c}> \mathbb{E}(D_{c,\tilde{c}}))$.
\begin{Lemma}
\label{lemma:theorem1}
Given the data generating model in Eq.~\eqref{eqa:model_decomp} equipped with the inner product specified by a predefined kernel function $k(\cdot, \cdot)$.
For a given distinct class $c$ and $\tilde{c}$, if $\exists \lambda < 1$, such that
\begin{equation}
\label{eqn:thm1}
\frac{\mathbb{E}\left(\|\bs{U}_{\tilde{c}}^T\bss{\varphi}_c\|_2^2\right)}{\mathbb{E}\left(\|\bs{U}_c^T\bss{\varphi}_c\|_2^2\right)}\leq \lambda
\end{equation}
where $\bss{\varphi}_c$ denotes a random vector from class $c$, then
\begin{eqnarray}
\label{eqa:bound}
\mathbb{P}(D_{c,c}>\mathbb{E}(D_{c,\tilde{c}}))\leq  \frac{1-\|\bs{m}_{c,\beta}\|_2^2}{1-\sqrt{{\lambda}}(1-\sigma_e^2)}
\end{eqnarray}
where $\bs{m}_{c,\beta}=\mathbb{E}\left(\bss{\beta}_c\right)$.
\end{Lemma}

\begin{proof}
Given multivariate random variables $\bss{\varphi}_c$, $\bss{\nu}_c$, $\bss{\kappa}_c$ drawn from class $c$ and $\bss{\mu}_{\tilde{c}}$ from class $\tilde{c}$. According to Eq.~\eqref{eqa:model_decomp}, we denote $\bss{\varphi}_c = \bs{U}_c\bss{\beta}_c+\bs{e}_c$, $\bss{\kappa}_c = \bs{U}_c\bss{\gamma}_c+\bs{e}_c$ and $\bss{\mu}_{\tilde{c}} = \bs{U}_{\tilde{c}}\bss{\eta}_{\tilde{c}}+\bs{e}_{\tilde{c}}$. We have:
\begin{eqnarray}
\label{eqa:proof0}
&& \mathbb{P}(D_{c,c}>\mathbb{E}(D_{c,\tilde{c}}))\leq \frac{\mathbb{E}(D_{c,c})}{\mathbb{E}(D_{c,\tilde{c}})} \\
\nonumber
& = &\frac{\mathbb{E}\left(\|\bss{\varphi}_{c} - \bss{\nu}_{c}\|_2^2\right)}{\mathbb{E}\left(\|\bss{\kappa}_{c} - \bss{\mu}_{\tilde{c}}\|_2^2\right)}\\
\nonumber
&=& \frac{\mathbb{E}\left(\|\bss{\varphi}_{c} - \mathbb{E}(\bss{\varphi}_{c}) - \bss{\nu}_{c} + \mathbb{E}(\bss{\varphi}_{c})\|_2^2\right)}{\mathbb{E}\left(\bss{\kappa}_{c}^T  \bss{\kappa}_{c}+ \bss{\mu}_{\tilde{c}}^T  \bss{\mu}_{\tilde{c}} - 2 \bss{\kappa}_{c}^T  \bss{\mu}_{\tilde{c}}\right)}\\
\nonumber
& \leq & \frac{2\mathbb{E}\left(\| \bss{\varphi}_{c} - \mathbb{E}(\bss{\varphi}_{c})\|^2_2\right)}{2 - 2\mathbb{E}(\bss{\kappa}_c^T\bss{\mu}_{\tilde{c}})}\\
\nonumber
& = & \frac{\mathbb{E}\left(\| \bss{\varphi}_{c} - \mathbb{E}(\bss{\varphi}_{c})\|^2_2\right)}{1 - \mathbb{E}(\bss{\kappa}_c^T\bss{\mu}_{\tilde{c}})}\\
\nonumber
&=&\frac{\mathbb{E}\left(\bss{\varphi}_{c}^T\bss{\varphi}_{c} - 2\bss{\varphi}_{c}^T\bs{U}_c\mathbb{E}\left(\bss{\beta}_c\right) + \bs{m}_{c,\beta}^T\bs{m}_{c,\beta} \right)}{1 - \mathbb{E}\left((\bs{U}_c\bss{\gamma}_c+\bs{e}_c)^T(\bs{U}_{\tilde{c}}\bss{\eta}_{\tilde{c}}+\bs{e}_{\tilde{c}})\right)}\\
\label{eqa:proof1}
&=&\frac{1-\|\bs{m}_{c,\beta}\|^2_2}{1 - \mathbb{E}(\bss{\gamma}_c^T\bs{U}_c^T\bs{U}_{\tilde{c}}\bss{\eta}_{\tilde{c}})}
\end{eqnarray}
Let $\bs{Q}_{c,\tilde{c}}=\bs{U}^T_c\bs{U}_{\tilde{c}}$. Eq.~\eqref{eqn:thm1} implies that:
\begin{eqnarray}
\nonumber
\frac{\mathbb{E}\left(\|\bs{U}_{\tilde{c}}^T\bss{\varphi}_c\|_2^2\right)}{\mathbb{E}\left(\|\bs{U}_c^T\bss{\varphi}_c\|_2^2\right)}&=& \frac{\mathbb{E}\left(\bss{\beta}_c^T\bs{U}_c^T\bs{U}_{\tilde{c}}\bs{U}_{\tilde{c}}^T\bs{U}_c\bss{\beta}_c\right)}{1-\sigma_e^2}\\
\nonumber
&=& \frac{\mathbb{E}\left(\bss{\beta}_c^T\bs{Q}_{c,\tilde{c}}\bs{Q}_{c,\tilde{c}}^T\bss{\beta}_c\right)}{1-\sigma_e^2}\leq \lambda\\
\label{eqa:proof2}
&\Rightarrow&\mathbb{E}\left(\|\bs{Q}_{c,\tilde{c}}^T\bss{\beta}_c\|_2^2 \right) \leq \lambda(1-\sigma_e^2)
\end{eqnarray}
Therefore, we have:
\begin{eqnarray*}
\eqref{eqa:proof1}
&=&\frac{1-\|\bs{m}_{c,\beta}\|^2_2}{1 - \mathbb{E}(\bss{\gamma}_c^T\bs{Q}_{c,\tilde{c}}\bss{\eta}_{\tilde{c}})}\\
&\leq &\frac{1-\|\bs{m}_{c,\beta}\|^2_2}{1 - \mathbb{E}(\|\bs{Q}_{c,\tilde{c}}^T\bss{\gamma}_c\|_2\|\bss{\eta}_{\tilde{c}}\|_2)}\\
&\leq&\frac{1-\|\bs{m}_{c,\beta}\|^2_2}{1 - \sqrt{\mathbb{E}(\|\bs{Q}_{c,\tilde{c}}^T\bss{\gamma}_c\|_2^2)\mathbb{E}(\|\bss{\eta}_{\tilde{c}}\|_2^2)}}\\
&\leq& \frac{1-\|\bs{m}_{c,\beta}\|^2_2}{1-\sqrt{\lambda}(1-\sigma_e^2)}
\end{eqnarray*}
\end{proof}
\begin{Remark}
\label{rmk:remark1}
There are three factors that contribute to the bound: $\|\bs{m}_{c,\beta}\|_2^2$, $\sigma_{e}^2$ and $\lambda$. It might appear that an increasing $\sigma_{e}^2$ leads to a lower value of the upper bound, which is counterintuitive. However, $\|\bs{m}_{c,\beta}\|_2^2$ and $\sigma_{e}^2$ are not independent. In fact, we have $\|\bs{m}_{c,\beta}\|_2^2\leq 1- \sigma_e^2$. An increasing $\sigma_{e}^2$ typically results in a smaller $\|\bs{m}_{c,\beta}\|_2^2$, which will increase the numerator in Eq.~\eqref{eqa:bound}.
Moreover, the term $1- \sigma_e^2$ is scaled by $\sqrt{\lambda}$, such that for $\sqrt{\lambda}\ll 1$, any change in $\sigma_{e}^2$ can be neglected compared to the corresponding difference in $\|\bs{m}_{c,\beta}\|_2^2$. When $\lambda\rightarrow 0$, the denominator goes to its maximum value $1$.
\end{Remark}
The result suggests that one heuristic of a good class separability is to have a feature space with small $\lambda$ and a large $\|\bs{m}_{c,\beta}\|_2^2$.
For given classes $\{c, \tilde{c}\}$ and a set of kernel functions $\mathcal{K} = \{k_1,\cdots,k_K\}$, one such criterion is the following:
\begin{eqnarray}
\label{eqa:criterion1}
k(\cdot,\cdot)=\underset{k_i\in \mathcal{K}}{\arg \min}\frac{\mathbb{E}\left(\|\bs{U}_{\tilde{c},i}^T\bss{\varphi}_{c,i}\|^2\right)}{\|\bs{m}_{c,\beta_i}\|_2^2}
\end{eqnarray}
where $\bss{\varphi}_{c,i}$ denotes a random vector from class $c$ in the feature space associated with kernel function $k_i$. For any class $r$, $\bs{U}_{r,i}$ is the matrix containing orthonormal basis vectors that span the subspace of noise free data from class $r$, where the feature space is associated with kernel function $k_i$. Moreover, $\bs{m}_{c,\beta_i}=\mathbb{E}\left(\bs{U}_{c,i}^T\bss{\varphi}_{c,i}\right)$.
Note that here we regard $\lambda$ (c.f. Eq.~\eqref{eqn:thm1}) to be $\frac{\mathbb{E}\left(\|\bs{U}_{\tilde{c},i}^T\bss{\varphi}_{c,i}\|^2\right)}{\|\bs{m}_{c,\beta_i}\|_2^2}$, since $\frac{\mathbb{E}\left(\|\bs{U}_{\tilde{c},i}^T\bss{\varphi}_{c,i}\|^2\right)}{\|\bs{m}_{c,\beta_i}\|_2^2}\geq\frac{\mathbb{E}\left(\|\bs{U}_{\tilde{c},i}^T\bss{\varphi}_{c,i}\|^2\right)}{\mathbb{E}\left(\|\bs{U}_{c,i}^T\bss{\varphi}_{c,i}\|^2\right)}$.
By using the criterion in Eq. \eqref{eqa:criterion1}, we are looking for a kernel function $k_i(\cdot,\cdot)$ associated with a small $\lambda$ and a large $\|\bs{m}_{c,\beta_i}\|_2^2$ simultaneously.

In the next theorem, let us generalize Lemma~\ref{lemma:theorem1} to the multiclass scenario.
\begin{Theorem}
\label{thm:theorem1}
Given the assumptions in Lemma~\ref{lemma:theorem1} and $\mathbb{P}(y=c)=p_c$, if $\exists \lambda < 1$, such that:
\begin{equation}
\label{eqn:thm2}
\frac{{{\sum}_{\forall c}\frac{p_c}{1-p_c}\sum}_{\forall \tilde{c}\neq c}p_{\tilde{c}}\mathbb{E}\left(\|\bs{U}_{\tilde{c}}^T\bss{\varphi}_c\|_2^2\right)}{{\sum}_{\forall c}p_c\mathbb{E}\left(\|\bs{U}_c^T\bss{\varphi}_c\|_2^2\right)}\leq \lambda
\end{equation}
then
\begin{eqnarray}
\label{eqa:bound2}
\mathbb{P}(D_w>\mathbb{E}(D_b))\leq  \frac{1-\sum_{\forall c}p_c\|\bs{m}_{c,\beta}\|^2_2}{1-\sqrt{\lambda}(1-\sigma_e^2)}
\end{eqnarray}
where
\begin{equation}
\begin{aligned}
D_{w}&& = & \|\bss{\gamma} - \bss{\tilde{\gamma}}\|_2^2,&& \forall \bss{\gamma}, \tilde{\bss{\gamma}} \text{ from the same class}\\
D_{b} && = &  \|\bss{\eta} - \bss{\tilde{\eta}}\|_2^2,&& \forall\bss{\eta}, \tilde{\bss{\eta}} \text{ from the different classes}
\end{aligned}
\end{equation}
are the within-class distance and the between-class distance for all classes, respectively.
\end{Theorem}
\begin{proof}
First, we show that
$$\small{\mathbb{E}\left(D_{w}\right)=\sum_{\forall c}p_c\mathbb{E}\left(D_{c,c}\right)}$$
and
$$\small{\mathbb{E}\left(D_{b}\right)=\sum_{\forall c} \sum_{\forall \tilde{c}\neq c}\frac{p_cp_{\tilde{c}}}{1-p_c}\mathbb{E}\left(D_{c,\tilde{c}}\right)}$$
Since they share the same principle, we only show the proof for the latter case.
Let $E_1$ denotes the event of drawing a random vector from class $\{1\cdots C\}$ and let $E_2$ be a second draw without replacement. Assume $E_1$ has Probability Mass Function $\mathbb{P}(E_1=c) = p_c$ with $\sum_{c=1}^Cp_c=1$ for $i=1,2$ and $c\in\{1\cdots C\}$. Then for a given class $c$ and any $\tilde{c}\in\{1\cdots C\}\setminus c$,
$$\mathbb{P}(E_2=\tilde{c}\mid E_1=c) =\frac{p_{\tilde{c}}}{1-p_c}$$
The PDF of $D_b$ can be then expressed as
\begin{eqnarray*}
\nonumber
f_{D_b}(t)&=&\sum_{\forall c}\sum_{\forall \tilde{c}\neq c}f_{D_b}(t\mid E_1=c, E_2=\tilde{c})\mathbb{P}(E_2=\tilde{c}| E_1 = c)\mathbb{P}(E_1=c)\\
&=&\sum_{\forall c}\sum_{\forall \tilde{c}\neq c}\frac{p_cp_{\tilde{c}}}{1-p_c}f_{D_b}(t\mid E_1=c, E_2=\tilde{c})\\
&=&\sum_{\forall c}\sum_{\forall \tilde{c}\neq c}\frac{p_cp_{\tilde{c}}}{1-p_c}f_{D_{c,\tilde{c}}}(t)
\end{eqnarray*}
The expected value is then computed as:
\begin{eqnarray}
\nonumber
\mathbb{E}\left(D_{b}\right) &= &\int tf_{D_b}(t)dt = \sum_{\forall c}\sum_{\forall \tilde{c}\neq c} \frac{p_cp_{\tilde{c}}}{1-p_c}\int tf_{D_{c,\tilde{c}}}(t) dt\\
\label{eqa:proof_lemma3_0}
 & = &\sum_{\forall c}\sum_{\forall \tilde{c}\neq c} \frac{p_cp_{\tilde{c}}}{1-p_c}\mathbb{E}(D_{c,\tilde{c}})
\end{eqnarray}
Similarly, we have $\mathbb{E}\left(D_{w}\right)=\sum_{\forall c}p_c\mathbb{E}(D_{c,c})$.

The rest of the proof can be readily extended from the proof of Lemma~\ref{lemma:theorem1} with the same setup. Since
\eqref{eqn:thm2} gives us $$\frac{{\sum}_{\forall c}\frac{p_c}{1-p_c}\sum_{\forall \tilde{c}\neq c}p_{\tilde{c}}\mathbb{E}\left(\|\bs{U}_{\tilde{c}}^T\bss{\varphi}_c\|_2^2\right)}{1-\sigma_e^2}\leq \lambda,$$
we obtain the following instead of Eq.~\eqref{eqa:proof2}:
\begin{eqnarray*}
\nonumber
&&{\left({\sum}_{\forall c}\frac{p_c}{1-p_c}{\sum}_{\forall \tilde{c}\neq c}p_{\tilde{c}}\mathbb{E}\left(\|\bs{Q}_{c,\tilde{c}}^T\bss{\beta}_c\|_2\right)\right)^2 }\\
\nonumber
&\leq& {{\sum}_{\forall c}\frac{p_c}{1-p_c}{\sum}_{\forall \tilde{c}\neq c}p_{\tilde{c}}\mathbb{E}\left(\|\bs{Q}_{c,\tilde{c}}^T\bss{\beta}_c\|_2^2 \right)}\leq \lambda(1-\sigma_e^2).
\end{eqnarray*}
Therefore, the following holds:
\begin{equation}
\label{eqa:proof3}
{\sum}_{\forall c}\frac{p_c}{1-p_c}{\sum}_{\forall \tilde{c}\neq c}p_{\tilde{c}}\mathbb{E}\left(\|\bs{Q}_{c,\tilde{c}}^T\bss{\beta}_c\|_2\right) \leq\sqrt{\lambda(1-\sigma_e^2)}
\end{equation}
Furthermore, similar to Eq.~\eqref{eqa:proof0} and \eqref{eqa:proof1}, we have:
\begin{eqnarray*}
&&\mathbb{P}(D_w>\mathbb{E}(D_b))\leq \frac{\mathbb{E}(D_w)}{\mathbb{E}(D_b)}\\
& = &\frac{{\sum}_{\forall c}p_c\mathbb{E}\left(\|\bss{\varphi}_{c} - \bss{\nu}_{c}\|_2^2\right)}{{\sum}_{\forall c}{\sum}_{\forall \tilde{c}\neq c}\frac{p_cp_{\tilde{c}}}{1-p_c}\mathbb{E}\left(\|\bss{\kappa}_{c} - \bss{\mu}_{\tilde{c}}\|_2^2\right)} \\
&\leq& \frac{1-\sum_{\forall c}p_c\|\bs{m}_{c,\beta}\|^2_2}{1 - {\sum}_{\forall c}{\sum}_{\forall \tilde{c}\neq c}\frac{p_cp_{\tilde{c}}}{1-p_c}\mathbb{E}(\bss{\gamma}_c^T\bs{U}_c^T\bs{U}_{\tilde{c}}\bss{\eta}_{\tilde{c}})}\\
&\leq&\frac{1-\sum_{\forall c}p_c\|\bs{m}_{c,\beta}\|^2_2}{1 - \underset{\forall c}{\sum}\underset{\forall \tilde{c}}{\sum}\frac{p_cp_{\tilde{c}}}{1-p_c}\mathbb{E}(\|\bs{Q}_{c,\tilde{c}}^T\bss{\gamma}_c\|_2)\mathbb{E}(\|\bss{\eta}_{\tilde{c}}\|_2)}\\
&\leq&\frac{1-\sum_{\forall c}p_c\|\bs{m}_{c,\beta}\|^2_2}{1 - \underset{\forall c}{\sum}\underset{\forall \tilde{c}}{\sum}\frac{p_cp_{\tilde{c}}}{1-p_c}\mathbb{E}(\|\bs{Q}_{c,\tilde{c}}^T\bss{\gamma}_c\|_2)\sqrt{1-\sigma_e^2}}\\
&\overset{Eq.~\eqref{eqa:proof3}}{\leq} & \frac{1-\sum_{\forall c}p_c\|\bs{m}_{c,\beta}\|^2_2}{1-\sqrt{\lambda}(1-\sigma_e^2)}
\end{eqnarray*}
\end{proof}
Theorem \ref{thm:theorem1} provides us an upper bound on the probability measure $\mathbb{P}(D_w>\mathbb{E}(D_b))$ that represents the class separability for multiclassification problems under the condition given by Eq.~\eqref{eqn:thm2}. In other words, to find the ``best'' kernel function according to Theorem \ref{thm:theorem1},
we can modify Eq.~\eqref{eqa:criterion1} into Criterion~\ref{c:criterion1}:
\begin{Criterion}
\label{c:criterion1}
Find a kernel function $k(\cdot,\cdot)$, such that
\begin{eqnarray}
\label{eqa:criterion2}
k(\cdot,\cdot)=\underset{k_i\in \mathcal{K}}{\arg \min}\frac{\sum_{\forall c} \sum_{\forall \tilde{c}\neq c} \frac{p_cp_{\tilde{c}}}{1-p_c}\mathbb{E}\left(\|\bs{U}_{\tilde{c},i}^T\bss{\varphi}_{c,i}\|_2^2\right)}{\sum_{\forall c}p_c\|\bs{m}_{c,\beta_i}\|_2^2}
\end{eqnarray}
where for the definition of $\bs{U}_{\tilde{c},i}$, $\bss{\varphi}_{c,i}$ and $\bs{m}_{c,\beta_i}$, one can refer to Eq.~\eqref{eqa:criterion1}.
\end{Criterion}

\subsection{CLAss-specific Subspace Kernel Functions}
\label{sec:clask}

Essentially, Eq.~\eqref{eqa:model_decomp} means that if we apply a nonlinear mapping $\varphi: \mathcal{X}\rightarrow \mathcal{F}$ associated with the given kernel function $k$, the noise free data from class $c$ span a subspace described by $\bs{U}_c$. Eq.~\eqref{eqa:model_decomp} implies that this is true for all classes in the feature space $\mathcal{F}$.
However, it is possible that this assumption holds for class $c$, but when transforming data vectors from another class $\tilde{c}\neq c$ to the feature space $\mathcal{F}$, we do not observe a subspace structure for class $\tilde{c}$. Nevertheless, if we apply another nonlinear transformation $\tilde{\varphi}: \mathcal{X}\rightarrow \tilde{\mathcal{F}}$, the subspace model in Eq~\eqref{eqa:model_decomp} holds true in $\tilde{\mathcal{F}}$ for data from class $\tilde{c}$. To deal with such an asymmetric case, we propose a class-specific feature map. Given a random data $\bs{x}_{c,i}$ from class $c$, the new nonlinear feature map is represented as:
\begin{eqnarray}
\label{eqa:model_decomp2}
\nonumber
{\varphi}(\bs{x}_{c,i}) &=& \begin{bmatrix}k_1({\cdot,\bs{x}_{c,i}})\\\vdots \\ k_C({\cdot,\bs{x}_{c,i}}) \end{bmatrix}=\begin{bmatrix}\bss{\varphi}_{(c,1),i}\\ \vdots \\\bss{\varphi}_{(c,C),i} \end{bmatrix} \\
\label{eqa:model_decomp3}
&=&  \begin{bmatrix}\bs{U}_{1} & 0 & 0\\ 0 & \ddots& 0\\ 0 & 0 &\bs{U}_{C}\end{bmatrix} \begin{bmatrix}\bss{\beta}_{(c,1),i}\\ \vdots\\\bss{\beta}_{(c,C),i} \end{bmatrix} + \begin{bmatrix}\bs{e}_{(c,1),i}\\ \vdots\\\bs{e}_{(c,C),i} \end{bmatrix}\\
\nonumber
&=&\bs{{U}}\bss{\beta}_{c,i}  + \bs{e}_{c,i} = \bss{\varphi}_{c,i}
\end{eqnarray}
where $k_{c'}(\cdot,\bs{x}_{c,i})$ is the class-specific nonlinear map that describes data from class $c'$ the ``best'' according to some optimality for all $c'\in\{1\cdots C\}$.

Note that we use notations in Eq.~\eqref{eqa:model_decomp2} without the sample index $i$ to denote their random variable counterparts. For example, for a random variable $\bs{x}_c$ from class $c$, $\bss{\varphi}_{(c,c')} \triangleq k_{c'}({\cdot,\bs{x}_{c}})$.

In this model, we have a class-specific kernel function $k_{c'}$ for each class $c'$.
Similar to Eq.~\eqref{eqa:model_decomp}, we assume that for all $c,c'\in\{1\cdots C\}$ 1) $\|\bss{\varphi}_{(c,c')}\|_2^2=1$; 2) $\bs{e}_{(c,c')}$ is zero-mean with $\mathbb{E}(\|\bs{e}_{(c,c')}\|_2^2)=\sigma_{e}^2$; and 3) Each dimension of vectors $\bs{U}_{c}\bss{\beta}_{(c,c')}$, and $\bs{e}_{(c'',c''')}$ are uncorrelated random variables $\forall c, c', c'', c'''\in \{1 \cdots C\}$.

The distance metric defined in Eq.~\eqref{eqa:distance} can be computed by:
\begin{equation}
d^2(\bss{\varphi}_c,\bss{\nu}_{c'}) = \sum_{c'' = 1}^C \|\bss{\varphi}_{(c,c'')} - \bss{\nu}_{(c',c'')}\|_2^2
\end{equation}

\begin{Theorem}
\label{thm:theorem2}
Given the class-specific kernel learning model in Eq.~\eqref{eqa:model_decomp2}, assume that $p_1=\cdots=p_C=\frac{1}{C}$. Let $\bs{m}_{(c,c')}=\mathbb{E}\left(\bss{\beta}_{(c,c')}\right)$ $\forall c,c'\in \{1\cdots C\}$. If $\exists\lambda<1$, such that:
\begin{equation}
\label{eqa:condition3}
{\frac{\frac{1}{C-1}\sum_{\forall c}\sum_{i\neq c}\mathbb{E}\left(\|\bs{U}_{i}^T\bss{\varphi}_{(c,i)}\|_2^2\right)}{\sum_{\forall c}\mathbb{E}\left(\|\bs{U}_{c}^T\bss{\varphi}_{(c,c)}\|_2^2\right)}\leq \lambda}
\end{equation}
then
\begin{equation}
\mathbb{P}(D_w > \mathbb{E}(D_b))\leq \frac{1- \frac{1}{C}\sum_{\forall c}\sum_{c'=1}^C \|\bs{m}_{(c,c')}\|_2^2}{ 1 - \frac{(C\lambda - \lambda +1 )(1-\sigma_e^2)}{C}}
\end{equation}
where $D_w$ and $D_b$ are the within-class distance and the between-class distance, respectively.
\end{Theorem}
\begin{proof}
The proof can be found in the appendix.
\end{proof}
Therefore, similar to the reasoning in Remark~\ref{rmk:remark1}, under the assumption of the class-specific kernel learning model, one heuristic is to find a kernel function according to the following criterion.
\begin{Criterion}
\label{c:criterion2}
Given Eq.~\eqref{eqa:model_decomp2} and a set of kernel functions $\mathcal{K}$, we want to find an optimal ordered set $\mathcal{K}^*=\{k_1\cdots k_C\}$, such that
\begin{eqnarray}
\label{eqa:criterion2}
\mathcal{K}^*=\underset{k_1\cdots k_C\in \mathcal{K}}{\arg \min}\frac{\frac{1}{C-1}\sum_{\forall c}\sum_{\forall \tilde{c}\neq c}\mathbb{E}\left(\|\bs{U}_{\tilde{c}}^T\bss{\varphi}_{(c,\tilde{c})}\|_2^2\right)}{\sum_{\forall c}\|\bs{m}_{(c,c)}\|_2^2}
\end{eqnarray}
\end{Criterion}
The solution to Criterion \ref{c:criterion2} can be found by exhaustive search, i.e. evaluating all possible combinations of the kernel functions to find the best for each class $c$. For $K$ given kernel functions and $C$ classes, this requires $K^C$ evaluations. An alternative is to find an optimal function $k^*_c$ for each class $c$, such that
\begin{eqnarray}
\label{eqa:criterion2_1}
k^*_c=\underset{k_c\in \mathcal{K}}{\arg \min}\frac{\frac{1}{C-1}\sum_{\forall \tilde{c}\neq c}\mathbb{E}\left(\|\bs{U}_{\tilde{c}}^T\bss{\varphi}_{(\tilde{c},c)}\|_2^2\right)}{\|\bs{m}_{(c,c)}\|_2^2}
\end{eqnarray}
This requires a one-step training with fixed rule and hence results in fast training.
In the next section, a more complex model based on optimization formulations is proposed using a class-specific multiple kernel function.

\subsection{The CLAss-Specific Multiple-Kernel model}

Multiple-Kernel (MK) learning models express the kernel function as a linear combination of $K$ different kernel functions $\{k_1\cdots k_K\}$:
\begin{equation}
k(\cdot,\bs{x}) =\sum_{l = 1}^K\nu_{l}k_l( \cdot,\bs{x})
\end{equation}
where $\nu_l$ is the weighting coefficient and $\sum_{l = 1}^K\nu_l=1$. Detailed motivations and descriptions can be found in \cite{Gonen_multiple_kernel_2011} and the references therein.

Given class-specific kernel functions and the subspace model in Eq. \eqref{eqa:model_decomp2}, we propose a class-specific multiple kernel to obtain higher flexibility.

\begin{Definition}[CLAss-Specific Multiple Kernel (CLASMK)]
Given a set of kernel functions $\mathcal{K}=\{k_1,\cdots, k_K\}$ and their corresponding basis matrices for each class, denoted by $\bs{U}_{c,k}$. Let $h_c:\mathcal{X}\times \mathcal{X}\rightarrow \mathbb{R}$ be defined as follows:
\begin{eqnarray}
\nonumber
h_c(\bs{x}, \bs{\tilde{x}})&=&\sum_{i=1}^K\sqrt{\nu_{c,i}} \varphi_i({\bs{x}})^T \bs{U}_{c,i}\sum_{j=1}^K\sqrt{\nu_{c,j}}\bs{U}_{c,j}^T\varphi_j(\tilde{\bs{x}})\\
 &&\text{for~ any~} \nu_{c,i}\geq 0 \text{~and~} \sum_{i=1}^K\nu_{c,i}=1
\end{eqnarray}
where $\varphi_i(\bs{z})=k_i(\cdot, \bs{z})$, for any $\bs{z}\in \mathcal{X}$.
The CLAss-Specific Multiple Kernel (CLASMK) is a Positive-Semi Definite (PSD) function $h(\bs{x},\tilde{\bs{x}})$ defined as:
\begin{equation}
\label{eqa:clasmk}
h(\bs{x},\tilde{\bs{x}}) = \sum_{c=1}^Ch_c\left(\bs{x},\tilde{\bs{x}}\right)
\end{equation}
\end{Definition}
By using the kernel function CLASMK according to Eq.~\eqref{eqa:clasmk}, the feature map can be explicitly written as
$$k_{CLASMK}(\cdot,\bs{x})=\begin{bmatrix}\sum_{i=1}^K\sqrt{\nu_{1,i}}\bs{U}_{1,i}^T\varphi_i({\bs{x}})\\ \vdots\\ \sum_{i=1}^K\sqrt{\nu_{C,i}}\bs{U}_{C,i}^T\varphi_i({\bs{x}})\end{bmatrix}$$
and linear models can be applied accordingly.
Hence, we are looking for a matrix $\bss{\nu}\in\mathbb{R}^{C\times K}$ according to the following criterion.
\begin{Criterion}[CLASMK]
\label{c:criterion3}
\begin{eqnarray}
\nonumber
\bss{\nu} &=&{\underset{\bss{\nu}}{\arg \min}\frac{\frac{1}{C-1}\sum_{\forall c}\sum_{\forall \tilde{c}\neq c}\mathbb{E}\left(h_{\tilde{c}}(\bs{x}_c, \bs{x}_c)\right)}{\sum_{\forall c}\bs{m}_{(c,c)}^T\bs{m}_{(c,c)}}}\\
\label{eqa:criterion3}
&& \text{s.t. }\sum_{i=1}^K{\nu}_{c,i}=1,~~ \forall c\\
\nonumber
&&~~~~~\nu_{c,i}\geq 0
\end{eqnarray}
where $\bs{m}_{(c,c)} = \mathbb{E}\left(\sum_{i=1}^K\sqrt{\nu_{c,i}} \bs{U}_{c,i}^T\varphi_i({\bs{x}})\right)$ for random vector $\bs{x}\in$ class $c$.
\end{Criterion}
\section{Algorithms and implementation}
\label{sec:algorithm}
In this section, first, we discuss the implementation given the CLAss-Specific Multiple-Kernel (CLASMK) model with respect to Criterion~\ref{c:criterion3}.

Given a set of kernel functions $\mathcal{K}=\{k_1,\cdots, k_K\}$ and a data set with balanced training set for each class, we use $\mathcal{D}_c$ and $N_c$ to denote the training set containing data from class $c$ and its sample size, respectively.
The problem is to find the weighting matrix $\bss{\nu}$, such that the approximation of Criterion~\ref{c:criterion3} is satisfied given the finite training set:
\begin{eqnarray}
\label{eqa:formulation1}
  \begin{aligned}
&\underset{\bss{\nu}}{\text{minimize}} &&\frac{\frac{1}{C-1}\underset{c=1}{\overset{C}{\sum}} \underset{\forall\tilde{c}\neq c}{\sum}\frac{1}{N_c}\underset{\forall \bs{x}\in \mathcal{D}_c}{\sum}h_{\tilde{c}}(\bs{x},\bs{x})}{\overset{C}{\underset{c'=1}{\sum}}\frac{1}{N_{c'}}\underset{\forall \bs{x}\in \mathcal{D}_{c'}}{\sum}h_{c'}(\bs{x},\bs{x})}
\\
&\text{subject~to}&&\sum_{i=1}^K\nu_{c,i} = 1, ~\forall c\\
&&& \nu_{c,k}\geq 0, ~\forall c, ~k
\end{aligned}
\end{eqnarray}
where $h_c$ is defined in Eq.~\eqref{eqa:clasmk}.

To achieve a robust result, we first divide the training data $\mathcal{D}$ into two disjoint subsets, i.e. $\mathcal{D}= \mathcal{M}_{B} \cup \mathcal{M}_{\nu}$.
The implementation for kernel model selection is twofold:
\begin{itemize}
\item[1)] Estimate the basis matrix $\bs{U}_{c,i}$ for each class $c$ and each kernel function $k_{i}$, $\forall i\in{1\cdots K}$, based on $\mathcal{M}_{B}$;
\item[2)] Find a matrix $\bss{\nu}$ that minimizes Eq.~\eqref{eqa:formulation1} based on $\mathcal{M}_{\nu}$.
\end{itemize}
Moreover, we use $\mathcal{M}_{B_c}$ and $\mathcal{M}_{\nu_c}$ to denote the corresponding matrices for each class $c$, respectively.

\subsection{Basis matrix $\bs{U}_{{c,k}}$}
\label{sec:basis}
Given Eq.~\eqref{eqa:model_decomp3}, the estimation of the basis matrices $\bs{U}_{c,k}$ can be formulated as Kernel Principle Component Analysis (KPCA) for data from each class $c$ associated with kernel function $k_{k}(\cdot,\cdot)$. To distinguish the basis matrix and its estimation, we use $\bss{\Phi}_{B_{c,k}}$ to denote the matrix containing the estimated principal components using the training data $\mathcal{M}_{B_c}$. However, when training size $N_c$ is large, we cannot carry out the full estimation of the covariance matrix in KPCA and hence \emph{kernel approximation} techniques (c.f.~Eq.~\eqref{eqa:kernel_approximation}) are needed. With both KPCA and kernel approximation, the problem is essentially finding 1) a matrix $\bs{X}_{B_{c,k}}$ and 2) a transformation matrix $\bs{A}_{c,k}$, such that for any data vector $\bs{x}_{c',i}$ from class $c'$, $\bss{\Phi}_{B_{c,k}}^T\varphi_k(\bs{x}_{c',i})  =\bs{A}_{c,k}^Tk_{k}\left(\bs{X}_{B_{c,k}}, \bs{x}_{c',i}\right)$ and $ \bss{\Phi}_{B_{c,k}}^T\bss{\Phi}_{B_{c,k}}= \bs{I}$, where columns in $\bs{X}_{B_{c,k}}$ are vectors \emph{sub-sampled} from $\mathcal{M}_{B_c}$ and $\bs{I}$ is the identity matrix.
This can be done using various standard techniques \cite{Scholkopf_nonlinear_component_1998,Diamantaras_principal_component_1996, Ouimet_greedy_spectral_2005}.

In this paper, a technique called the Greedy Spectral Embedding algorithm \cite{Ouimet_greedy_spectral_2005} is adopted. It selects a subset for kernel representation by computing the subspace innovation and then find the basis matrix for the subspace spanned by the subset.
However, in order to have class-specific subspaces, we modify the algorithm to find a basis matrix for each individual class. In fact, this block based structure results in a lower computational complexity.

\subsection{Kernel function $k_c$}
After obtaining the matrix $\bss{\Phi}_{B_{c,k}}$, parameterized by $\bs{X}_{B_{c,k}}$ and $\bs{A}_{c,k}$, we can proceed to finding the optimal kernel function for each class. Two models have been introduced in this section: the CLAss-specific Subspace Kernel (CLASK) model and the CLAss-Specific Multiple-Kernel (CLASMK) model.
\subsubsection{CLASK-ML: Finding the best kernel}
When CLASK model is applied without Multiple-Kernel, one can find the best kernel function for each class according to Eq.~\eqref{eqa:criterion2_1}.

\subsubsection{CLASMK-ML: Solving for $\bss{\nu}$}

Empirically, given any subset $\mathcal{M}_{\nu_c}\subset \mathcal{D}_c\setminus \mathcal{M}_{B_c}$, we rewrite Eq. \eqref{eqa:formulation1} as:
\begin{eqnarray}
\label{eqa:formulation2}
\begin{aligned}
&\underset{\bss{\nu}}{\text{minimize}} &&\frac{\frac{1}{C-1}\underset{c=1}{\overset{C}{\sum}} \underset{\forall\tilde{c}\neq c}{\sum}\frac{1}{N_c}\underset{\forall \bs{x}\in \mathcal{M}_{B_c}}{\sum}\|\underset{i=1}{\overset{K}{\sum}}\sqrt{\nu_{\tilde{c},i}}\bs{A}_{\tilde{c},i}^Tk_i\left( \bs{X}_{B_{\tilde{c},i}},{\bs{x}}\right)\|_2^2}
{\overset{C}{\underset{c=1}{\sum}}\frac{1}{N_c}\underset{\forall \bs{x}\in \mathcal{M}_{B_c}}{\sum}\|\underset{i=1}{\overset{K}{\sum}}\sqrt{\nu_{c,i}}\bs{A}_{c,i}^Tk_i\left( \bs{X}_{B_{c,i}},{\bs{x}}\right)
\|_2^2}\\
&\text{subject ~to}&&\sum_{k=1}^K\nu_{c,k} = 1, ~\forall c\\
&&& \nu_{c,k}\geq 0, ~\forall c, ~k
\end{aligned}
\end{eqnarray}
and apply a nonlinear optimization technique to approximate the optimal solution of $\bss{\nu}$.
\subsection{Summary of the Algorithm}
Given training data $\mathcal{D}$ and a set of kernel functions, the algorithm for finding the weighting coefficients is summarized in Algorithm \ref{alg:CLASMK}.
\SetKwInOut{Parameter}{Parameters}
\begin{algorithm}[ht!]
 \caption{CLAss-Specific Multiple-Kernel Metric Learning (CLASMK-ML)}
 \KwIn{Training data $\mathcal{D}$ and a set of kernel functions: $\mathcal{K}=\{k_{i}\}_{i=1}^K$. Assume that the training size for all classes are balanced.
Hyperparameters: $\eta$ (Eq.~\eqref{eqa:truncation}) and $t$ (Eq.~\eqref{eqa:threshold_t})\;
}
  \KwOut{Kernel function weighting coefficient matrix $\bss{\nu}$\;}
 {\bf Initialization}: Let $\nu_{c,k}= \frac{1}{K}$, $\forall c,k$\;
Divide the dataset into two subsets $\mathcal{M}_{B}$ and $\mathcal{M}_{\nu}$ for estimating the basis and $\bss{\nu}$, respectively.\\
Let $\mathcal{M}_{B_c}\subset\mathcal{M}_{B}$ be the subset that contains all training data from class $c$.\;
\For{$c\in\{1\cdots C\}$} {
\For{$k \in \{1\cdots K\}$} {
Basis estimation based on $\mathcal{M}_{B_c}$ using KPCA or any iterative methods (c.f. Sec.~\ref{sec:basis}) with respect to kernel function $k_{k}(\cdot,\cdot)$ for every class $c$:
\begin{itemize}
\item Find the matrix $\bs{X}_{B_{c,k}}$.
\item Find the transformation matrix $\bs{A}_{c,k}$.
\end{itemize}
}
}
Truncate the non-representative kernels according to Sec.~\ref{sec:truncation}\;
Find $\bss{\nu}$ by solving Eq.~\eqref{eqa:formulation2} using all $(\bs{x},t)\in \mathcal{M}_{\nu}$\;
\vspace{5mm}
\label{alg:CLASMK}
\end{algorithm}

\subsection{Remarks}
\subsubsection{Truncating Non-representative Kernel Functions}
\label{sec:truncation}
For a given kernel function $k_i$ from the set $\mathcal{K}$, the objective function in Eq.~\eqref{eqa:criterion3} can be empirically estimated as follows:
\begin{equation}
h = \frac{h_b}{h_w} =\frac{\frac{1}{C-1}\underset{c=1}{\overset{C}{\sum}} \underset{\forall\tilde{c}\neq c}{\sum}\frac{1}{N_c}\underset{\forall \bs{x}\in \mathcal{M}_{\nu_c}}{\sum}\|\bs{A}_{\tilde{c},i}^Tk_i\left( \bs{X}_{B_{\tilde{c},i}},{\bs{x}}\right)\|_2^2}
{\overset{C}{\underset{c=1}{\sum}}\frac{1}{N_c}\underset{\forall \bs{x}\in \mathcal{M}_{\nu_c}}{\sum}\|\bs{A}_{c,i}^Tk_i\left( \bs{X}_{B_{c,i}},{\bs{x}}\right)
\|_2^2}
\end{equation}
Generally speaking, the smaller $h$ is, the better class separating ability $k_i$ have. However, in some scenarios, although $h\approx 0$, the feature space associated with kernel function $k_i$ results in a poor class separation. This is due to the possibility of non-representative subspaces, i.e. the estimated basis cannot represent the data vectors $\bs{x}\in \mathcal{M}_{\nu_c}$, which implies that $h_w\approx 0$.

Due to the small value of $h\approx 0$, such phenomenon severely influences the outcome of the optimization result. It often occurs for kernel functions with high complexities, such as RBF kernels with small $\sigma$.
To resolve this issue, we insert a pre-truncating step before the optimization.
A kernel is defined as non-representative if
\begin{equation}
\label{eqa:truncation}
{\overset{C}{\underset{c=1}{\sum}}\frac{1}{N_c}\underset{\forall \bs{x}\in \mathcal{M}_{\nu_c}}{\sum}\|\bs{A}_{c,i}^Tk_i\left( \bs{X}_{B_{c,i}},{\bs{x}}\right)
\|_2^2}\leq \eta
\end{equation}
for given hyperparameter $\eta\in (0,1]$. Note that $\bs{A}_{c,i}$ and $\bs{X}_{B_{c,i}}$ are trained using subset $\mathcal{M}_{B_c}\subset \mathcal{D}_c$, whereas Eq.~\eqref{eqa:truncation} is evaluated using $\mathcal{M}_{\nu_c}\subset \mathcal{D}_c\setminus \mathcal{M}_{B_c}$.
Then the truncation is performed according to:
\begin{equation}
\mathcal{K}\leftarrow \mathcal{K}\setminus \mathcal{\tilde{K}}
\end{equation}
where $\mathcal{\tilde{K}}$ contains all kernel functions that are non-representative.

\subsubsection{Robustness and Empirical Thresholding}
\label{sec:remark2}
We model data as random vectors that are drawn from a unknown multivariate probability distribution. Given different training sets, the solutions $\bss{\nu}$ computed using Algorithm \ref{alg:CLASMK} result in a certain variation, which gives rise to the notion of robustness. For a given sample size, we attempt to improve robustness by empirical thresholding, i.e. given $\nu_{c,k}$ for all $c=\{1\dots C\}$ and $k=\{1\dots K\}$ and a threshold $t$, such that
\begin{equation}
\label{eqa:threshold_t}
\tilde{\nu}_{c,k}=\begin{cases}\nu_{c,k},  & \text{if} ~~ \nu_{c,k}\geq t\underset{\forall k}{\max}\nu_{c,k} \\ 0, & \text{otherwise}\end{cases}
\end{equation}
and then the new weights are re-scaled as follows:
\begin{equation}
\label{eqa:nu_threshold}
\nu_{c,k} \leftarrow  \frac{\tilde{\nu}_{c,k}}{\sum_{i=1}^K\tilde{\nu}_{c,i}}
\end{equation}

\subsubsection{Finding the Best Kernel Function}
\label{sec:best_kernel}
Algorithm \ref{alg:CLASMK} can also be used as an approximation to the CLASK-ML formulation presented in Criterion \ref{c:criterion2}. In that case, given the output of Algorithm \ref{alg:CLASMK}, denoted by $\bss{\nu}$, we choose the optimal kernel function $k^*_c$ for each class as follows. Let $\{k^{(1)},\cdots, k^{(K)}\}$ be the set of kernel functions, such that each $k^{(i)}$ is corresponding to $\bss{\nu}(:,i)$. Then $k^*_c = k^{(i_c^*)}$, where
\begin{equation}
\label{eqa:k_star}
i_c^* = \underset{i}{\arg\max}~ \bss{\nu}(c,i)
\end{equation}

\section{Learning Hierarchical CLASMK Feature Network}
\label{sec:hierarchy}

In this section, we propose a multi-layer kernel feature network constructed using the CLAss-specific Subspace Multiple Kernel - Metric Learning (CLASK-ML) on each layer. The structure of the network can be found in Fig.~\ref{fig:hierarchy}, where each layer $l$ corresponds to a feature space $\mathcal{F}_{(l)}^C$ constructed using CLASMK-ML (c.f.~Algorithm~\ref{alg:CLASMK}) for a given set of kernel functions $\mathcal{K}$.
For each layer, the idea is to select a subset of training data with small margins and learn a new feature space based on that. The final feature space is the concatenation of the resulting features from all the previous layers.

The learning process is explained as follows:
\begin{itemize}
\item Input: Training set $\mathcal{D}$, a set of kernel functions $\mathcal{K}$ ($\left|\mathcal{K}\right|=K$), base classifier $f(\bs{z}; \bs{W},\bs{b})=\bs{W}^T\bs{z}+\bs{b}$, parameterized by $\bs{W}\in\mathbb{R}^{n\times C}$, $\bs{z}\in\mathbb{R}^n$, $n\in\mathbb{N}^+$, $\bs{b}\in\mathbb{R}^C$;
\item Hyperparameters: Maximum number of layers $L_{\max}$, marginal threshold $T_{\kappa}\in(0,1)$, hyperparameters for Algorithm~\ref{alg:CLASMK}, $t,\eta\ll 1$;
\item Procedure:
\begin{itemize}
\item Initialization:
\begin{itemize}
\item $\mathcal{M}\triangleq\mathcal{D}$;
\item $d_{\bss{\nu}}=1$, $\delta_-=0$;
\item $q = 1$;
\item $\bss{\nu}^{(0)}(i,j)=1/K$, $\forall i\in\{1\cdots C\}, ~\forall j\in\{1\cdots K\}$;
\item $\bss{\psi}_i=$ empty array $\forall i\in\mathcal{I}_{\mathcal{D}}$;
\item $L=1$.
\end{itemize}
\item \text{while}$\left(L\leq L_{\max}~\text{and}~ d_{\bss{\nu}}>\epsilon \right)$
\begin{itemize}
\item Apply Algorithm~\ref{alg:CLASMK} to compute $\bss{\nu}^{(l)}$ on $\mathcal{M}$;
\item For each training data $i\in\mathcal{I}_{\mathcal{D}}$, construct feature vector:
\begin{equation}
\bss{\psi}^{(l)}_i\triangleq
\begin{bmatrix}
\sum_{i=1}^K\sqrt{\nu^{(l)}_{1,k}} \mathbf{U}_{1,k}^T\varphi_1(\mathbf{x}_i)\\
\vdots\\
\sum_{k=1}^K\sqrt{\nu^{(l)}_{C,k}} \mathbf{U}_{C,k}^T\varphi_C(\mathbf{x}_i)
\end{bmatrix}
\end{equation}
and
\begin{equation}
\bss{\psi}_i\leftarrow
\begin{bmatrix}
\bss{\psi}_i\\
\bss{\psi}_i^{(l)}
\end{bmatrix}
\end{equation}
\item $q = dim\left(\bss{\psi}_i\right)$;
\item Train the base classifier using all $i\in\mathcal{I}_{\mathcal{D}}$
\begin{equation}
\{\bs{W}^*,\bs{b}^*\} = \underset{\bs{W},\bs{b}}{\arg\min}~\mathcal{R}\left(f(\bss{\psi}_i;\bs{W},\bs{b})\right)
\end{equation}
where $\mathcal{R}$ is the predefined empirical risk.
\item Compute the slack variables \cite{Vapnik_the_nature_1995} for each data $i$:
\begin{equation}
\bss{\xi}_i(c) = \bs{W}(:,c)^T\bss{\psi}_i + \bs{b}(c)
\end{equation}
where $c\in\{1\cdots C\}$; $\bs{W}(:,c)$ and $\bs{b}(c)$ represent the $c^{th}$ column and the $c^{th}$ element of matrix $\bs{W}$ and $\bs{b}$, respectively.
\item Check the confidence level $\kappa_i$ of each data $i$
\begin{equation}
\label{eqa:confidence}
\kappa_i \triangleq \frac{\left(\underset{c}{\max}(\bss{\xi}_i(c)) - \underset{j\neq c}{\max}(\bss{\xi}_i(j))\right)}{2}
\end{equation}
\item Select the ``marginal'' subset:
\begin{equation}
\label{eqa:marginal}
\mathcal{I}_{\mathcal{M}}\leftarrow \begin{Bmatrix}i: \kappa_i\leq T_{\kappa}\end{Bmatrix}
\end{equation}
\item $L\leftarrow L +1 $;
\item $\delta_{+} = \left|\left|\bs{\nu}^{(L)}-\bs{\nu}^{(L-1)}\right|\right|_F$;
\item $d_{\bss{\nu}}= \left|\delta_{+}-\delta_{-}\right|$;
\item $\delta_-\leftarrow \delta_+$;
\end{itemize}
\end{itemize}
\item Output: $\{\bss{\nu}^{(1)}\cdots \bss{\nu}^{(L)}\}$; $\bss{\psi}_i$, $\forall i\in\mathcal{I}_{\mathcal{D}}$.
\end{itemize}

\begin{Remark}[Confidence $\kappa$]
Given Eq.~\eqref{eqa:multiclass_xi} and Eq.~\eqref{eqa:multiclass_t},
ideally speaking, for a data point $\bs{x}_i\in$ class $c$, we expect the following:
\begin{eqnarray*}
\underset{j}{\arg\max}(\bss{\xi}_i(j))&=&c\\
\bss{\xi}_i(c)&\geq&1\\
\bss{\xi}_i(\tilde{c})&\leq&-1,~\forall \tilde{c}\neq c
\end{eqnarray*}
Hence, the quantity $\kappa$ (c.f.~Eq.~\eqref{eqa:confidence}) is considered as a confidence measure and we call data $\bs{x}_i$ ``marginal'' if $-1<\kappa_i<1$.
By training on the marginal dataset, we focus on the ``uncertain'' area and the CLASMK can be designed accordingly on the corresponding layer.
\end{Remark}

\begin{figure}
\begin{center}
\includegraphics[width=120mm]{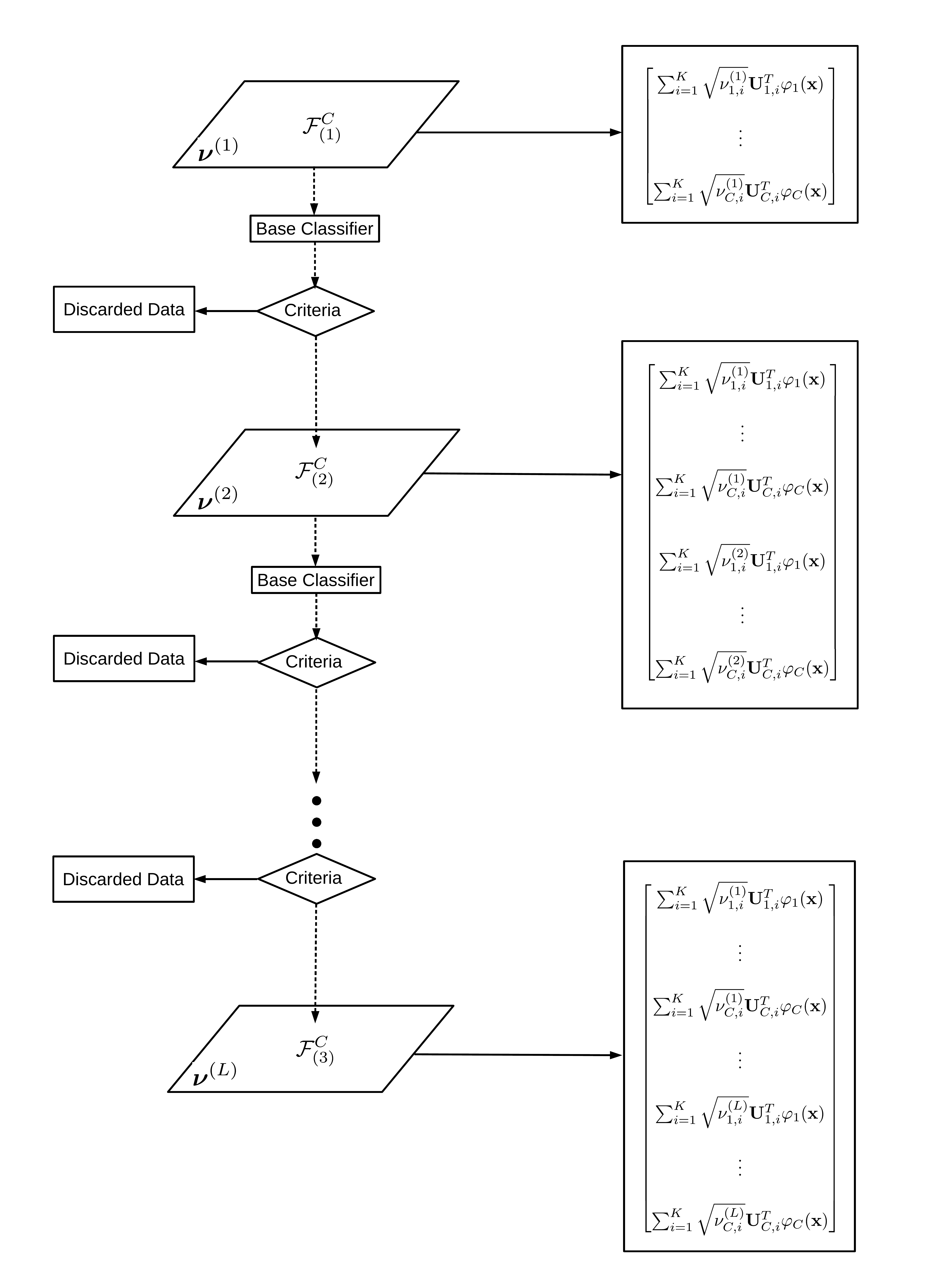}
\caption{The structure of the hierarchical feature network. At each layer $l$, a new CLASMK is learned using the learning algorithm CLASKMK-ML presented in Algorithm \ref{alg:CLASMK}. The feature vectors at $l$ is the concatenation of the outputs from all previous layers. After the feature space is constructed, a predefined base classifier is applied to select the ``marginal'' subset of the training data, which is then passed to the next layer for learning $\bss{\nu}^{(l+1)}$.}
\label{fig:hierarchy}
\end{center}
\end{figure}

\section{Experimental Results}
\label{sec:results}

The experimental results consist of the following:
\begin{itemize}
\item[\ref{sec:one_layer}] {\bf One layer CLASK-ML and CLASMK-ML compared to single kernel learning:}
The purpose is to verify the gain of using learning class-specific kernel functions using CLASK-ML and CLASMK-ML algorithms compared to learning with a single kernel.
\item[\ref{sec:two_layer}] {\bf Multi-layer CLASMK-ML compared to other MK techniques:}
In this section, we compare the multi-layer CLASMK-ML to the state-of-the-art techniques on standard datasets. To simplify the comparison, we follow \cite{zhuang_two_layer_2011} and use 2-layer network for comparison.
\item[\ref{sec:multi_layer}] {\bf Classification accuracy improvement with respect to the number of layers:}
We investigate how the following items effect the classification performance (accuracy, training time, resulting dimensionality) using the hierarchical CLASMK-ML algorithm described in Sec.~\ref{sec:hierarchy}:
\begin{itemize}
\item number of layers;
\item number of kernel functions;
\item training size.
\end{itemize}
\end{itemize}
The hyperparameters are chosen to be $L_{\max}=10$, $T_{\kappa}=1$, $\eta=0.1$ and $t = 0.1$.

\subsection{One Layer CLASK-ML and CLASMK-ML Compared to Single Kernel Learning}
\label{sec:one_layer}
In this section, we use the well known radial basis function (RBF) and the polynomial kernels with different values of the hyperparameters.
To verify the efficiency of Algorithm \ref{alg:CLASMK}, we compare the classification performance with respect to the following schemes:
\begin{itemize}
\item[1)] Standard kernel approximation with eight different kernel functions, including:
\begin{itemize}
\item RBF kernel $k_{\text{RBF}}(\bs{x},\bs{\tilde{x}}) = \exp(\frac{\|\bs{x}-\tilde{\bs{x}}\|_2^2}{\sigma^2})$ with hyperparameters $\sigma\in \begin{Bmatrix}10, 1, 0.1, 0.01 \end{Bmatrix}$;
\item polynomial kernel $k_{\text{POLY}}(\bs{x},\bs{\tilde{x}}) = (1+\bs{x}^{\text{T}}\tilde{\bs{x}})^d$ with hyperparameters $d\in \begin{Bmatrix}8, 12, 24, 48 \end{Bmatrix}$;
\end{itemize}
\item[2)] CLAss-specific Subspace Kernel Metric Learning (CLASK-ML) using Algorithm \ref{alg:CLASMK} with the ``best'' kernel function for each class selected from set of all given kernel functions $\mathcal{K}$ according to Eq.~\eqref{eqa:k_star};
\item[3)] CLAss-Specific Multiple Kernel Metric Learning (CLASMK-ML) using Algorithm \ref{alg:CLASMK} with kernel functions from set $\mathcal{K}$ according to Eq.~\eqref{eqa:clasmk}.
\end{itemize}
\paragraph{Configuration and hyperparameters:}
For the choice of implementations and hyperparameters in this experiments, we list the following:
\begin{itemize}
\item Optimization technique: To approximate the optimal solution in Eq.~\eqref{eqa:formulation2}, we implement the algorithm using the MATLAB (2015b) built-in function ``fmincon'' with the optimizer (``interior-point technique'').
\item Hyperparameters $\eta$ (Eq.~\eqref{eqa:truncation}) and $t$ (Eq.~\eqref{eqa:threshold_t}) are heuristically chosen to be $0.01$ and $10^{-4}$, respectively, for truncating the obtained optimal $\bss{\nu}$ given training data.
\end{itemize}

\paragraph{Experiments and results:}
The classifier used to evaluate the performance is the classical multiclass LS-SVM \cite{Suykens_least_1999} due to its simplicity.
The classification results on the testing data are listed in Table~\ref{tab:result} for various datasets \cite{uci}. These results are obtained by 10-fold validation, where at each iteration, 10\% data are used for testing and the rest for training. The presented results include the mean value of the classification error on the testing set $\pm$ the standard deviation.
The result is organized as follows. For each dataset, we first test each single kernel function from the set $\mathcal{K}$ of eight kernel functions. The best results are marked as bold letters. The algorithm CLASK-ML is then applied to identify the best kernel function amongst members of $\mathcal{K}$ according to Sec. \ref{sec:best_kernel}. Finally, we use the CLASMK-ML algorithm to determine the best combination of all kernel functions.
We can see that CLASMK-ML provides the best classification results.
In Table~\ref{tab:time}, the size and training time for each dataset are listed. Kernel approximation using CLASK model is faster to train due to the block-based structure (c.f.~Eq.~\eqref{eqa:model_decomp2}), which results in a reasonable overall training time.
\begin{table*}[ht!]
\begin{center}
    \begin{tabular}{ | c | c | c | c | c | }
    \hline
    Dataset & \multicolumn{4}{|c|}{LS-SVM Classification Error (\%) Using Different Kernels} \\ \hline
     \multirow{5}{*}{Banana} & poly (8) &  poly (12) &   poly (24) & { poly (48)}\\
& $29.53\pm 2.92$& $21.00\pm 2.59$  & $ 14.09\pm 1.97$ & ${ 11.04\pm 1.05}$      \\\cline{2-5}
& RBF (1) &  RBF (0.5) &  RBF (0.1) &  {\bf RBF (0.05)} \\
& $35.23\pm 4.52$& $17.35\pm3.14 $  & $ 10.04\pm 1.99 $ & ${\bf 9.98\pm 1.21} $       \\\cline{2-5}
\multirow{5}{*}{} & \multicolumn{2}{|c|}{ CLASK-ML $ 10.02 \pm 1.08$} & \multicolumn{2}{|c|}{ {\bf CLASMK-ML} ${\bf 9.36 \pm 1.01}$} \\ \hline\hline
\multirow{5}{*}{Pendigits}& poly (8) &  poly (12) & poly (24) &  {\bf  poly (48)} \\
& $2.49\pm 1.30$ & $1.57\pm 0.95$ & $0.98\pm 0.68$ & $ {\bf 0.67\pm 0.33 }$  \\\cline{2-5}
& RBF (1) &  RBF (0.5) &  RBF (0.1) &  RBF (0.05)\\
& $3.31 \pm 0.90 $& $ 3.39\pm 0.92$  & $ 3.39\pm 0.92 $ & $35.21\pm 4.90 $     \\\cline{2-5}
\multirow{5}{*}{} & \multicolumn{2}{|c|}{ {\bf CLASK-ML} ${\bf  0.45 \pm 0.35} $} & \multicolumn{2}{|c|}{ {\bf CLASMK-ML} $ {\bf 0.42 \pm 0.31} $ }\\ \hline\hline
     \multirow{5}{*}{Optdigits}& poly (8) &  { poly (12)} & {\bf  poly (24)} &  poly (48)\\
&  $4.25\pm 2.38$&  ${ 3.76 \pm 1.86}$ & ${\bf 3.01\pm 1.57}$ & $3.28\pm 1.89$  \\\cline{2-5}
& RBF (1) &  RBF (0.5) &  RBF (0.1) &  RBF (0.05) \\
& $3.60\pm 1.60 $& $3.23 \pm 1.88 $  & $89.41 \pm 1.32$ & $90.00\pm 0.81$ \\\cline{2-5}
\multirow{5}{*}{} & \multicolumn{2}{|c|}{ {\bf CLASK-ML} $ {\bf 3.01 \pm 1.48}$} &\multicolumn{2}{|c|}{  {\bf CLASMK-ML} ${\bf  2.96 \pm 1.45}$} \\ \hline\hline
     \multirow{5}{*}{Phoneme}& poly (8) &  poly (12) &  poly (24) &  poly (48)\\
&  $16.13\pm 2.06$&  $15.38\pm 1.70$ & $13.43\pm 1.63$ & $12.82\pm 0.98$  \\\cline{2-5}
&  RBF (1) &  RBF (0.5) &  { RBF (0.1)} & {\bf RBF (0.05)} \\
& $18.70\pm 1.87 $  & $15.11 \pm 1.34$ & ${ 10.93\pm 1.09}$ & ${\bf 10.18\pm 0.96}$ \\\cline{2-5}
\multirow{5}{*}{} & \multicolumn{2}{|c|}{ {\bf CLASK-ML} $ {\bf 9.93\pm 0.91 }$} & \multicolumn{2}{|c|}{ {\bf CLASMK-ML} ${\bf  9.93\pm 0.91}$} \\ \hline
    \end{tabular}
\caption{Experimental results on various datasets using 10-fold validation. The results are presented as the averaged classification error $\pm$ its standard deviation using different kernel functions. The base classifier is the LS-SVM. }
\label{tab:result}
\end{center}
\end{table*}

We observe that the training time for kernel approximation (c.f.~Sec.~\ref{sec:basis}) increases when using RBF kernels with a small $\sigma$. The reason is twofold. First, the computation for RBF kernels is more complex compared to polynomial kernels. Secondly, when $\sigma$ is small, it means that the correlation between data points in the kernel induced feature space is small.
Hence, for the same tolerance $\|\bs{K}-\bs{L}\|$ (c.f.~Eq.~\eqref{eqa:kernel_approximation}), more data points are needed for constructing $\bs{L}$.

\begin{table}[ht!]
\begin{center}
    \begin{tabular}{ | c | c | c | c | c | }
    \hline
    Dataset (Size) & \multicolumn{4}{|c|}{Training Time Using Different Kernel Functions} \\ \hline
     \multirow{2}{*}{Banana} & poly (8) &  poly (12) &   poly (24) & { poly (48)}\\
& $0.60$ & $0.61$  & $ 0.73$ & $1.24$      \\\cline{2-5}
\multirow{3}{*}{($2\times 5300$)}& RBF (1) &  RBF (0.5) &  RBF (0.1) &  {RBF (0.05)} \\
 & $0.040$& $ 0.041 $  & $ 0.09 $ & $0.41 $       \\\cline{2-5}
\multirow{5}{*}{} &\multicolumn{4}{|c|}{CLASMK-ML  $75.95$ $(2.70+73.25)$} \\ \hline\hline
\multirow{2}{*}{Pendigits}& poly (8) &  poly (12) & poly (24) &  {  poly (48)} \\
& $3.65$ & $6.723$ & $19.72$ & $ { 63.74}$  \\\cline{2-5}
 \multirow{3}{*}{($16\times 3498$)}& RBF (1) &  RBF (0.5) &  RBF (0.1) &  RBF (0.05)\\
& $0.25$& $9.82 $  & $ 275.32 $ & $292.23 $     \\\cline{2-5}
\multirow{5}{*}{} &\multicolumn{4}{|c|}{CLASMK-ML  235.53 $(226.8+8.73)$}\\ \hline\hline
     \multirow{2}{*}{Optdigits}& poly (8) &  { poly (12)} & poly (24)&  poly (48)\\
&  $9.63$&  $22.04$ & $44.21$ & $51.24$  \\\cline{2-5}
\multirow{3}{*}{($64\times 1797$)} & RBF (1) &  RBF (0.5) &  RBF (0.1) &  RBF (0.05) \\
& $20.02 $& $38.94 $  & $29.01 $ & $35.82$ \\\cline{2-5}
\multirow{5}{*}{} &\multicolumn{4}{|c|}{CLASMK-ML  $33.74$ $(27.51+6.23)$} \\ \hline\hline
     \multirow{2}{*}{Phoneme}& poly (8) &  poly (12) &  poly (24) &  poly (48)\\
&  $1.82$&  $2.91$ & $6.08$ & $11.62$  \\\cline{2-5}
\multirow{3}{*}{($5\times 5404$)} &  RBF (1) &  RBF (0.5) &  { RBF (0.1)} & {RBF (0.05)} \\
& $0.06$  & $0.10$ & $76.24$ & $595.32$ \\\cline{2-5}
\multirow{5}{*}{} & \multicolumn{4}{|c|}{CLASMK-ML $594.22$ $(73.55+520.67)$} \\ \hline
    \end{tabular}
\caption{This table summarizes the datasets and their training time. With a single predefined kernel function, the training time is the kernel approximation. For CLASMK-ML, it is the training time of (kernel approximation + metric learning). Note that the kernel approximation for CLASMK-ML is the summation of the computational time for all eight kernel functions.}
\label{tab:time}
\end{center}
\end{table}

\subsection{Multi-Layer CLASMK-ML Compared to Other Multi-Layer MK Techniques}
\label{sec:two_layer}
We have compared the hierarchical CLASMK-ML to other multi-layer multiple-kernel learning techniques on 12 standard datasets \cite{varma_more_generality_2009, Kloft_efficient_and_2008,  Gehler_inifite_kernel_2008, Cho_kernel_methods_2009, Xu_an_extended_2008}. The description of the datasets can be found in Table.~\ref{tab:datasets_2class}. We follow the exact same setup as in \cite{zhuang_two_layer_2011, Xu_an_extended_2008}.
For each data set, we use 13 kernel function including (c.f.~Sec.~\ref{sec:one_layer}) $k_{\text{RBF}}$ with $\sigma\in\{2^{-3}, 2^{-2},\cdots,2^{6}\}$ and $k_{\text{POLY}}$ with $d\in \{1,2,3\}$.
We use 50\% data randomly selected from all instances as training data, and the rest for testing. The preprocessing on the training
data is to remove the mean value and normalize to unit
variance. The same preprocessing is done on the test instances with the same mean and variance computed from the training data.
We repeat 20 times and report the sample mean and the standard deviation of the result in Table~\ref{tab:accuracy}.
The average of the results are summarized on the last row in the table.

We observe that out of the twelve datasets, with a two layer structure, CLASMK-ML outperforms 2L-MKL on six datasets and has equivalent results on three datasets. On average, 2L-CLASMK-ML achieves a similar result compared to 2L-MKL.

\begin{table}[ht!]
\begin{center}
    \begin{tabular}{ | c | c | c |  c | c | c| c|}
    \hline
    Dataset   & Diabetes & Breast &    Australian  & { Titanic} & Ionosphere &Banana \\\hline
 $(p, N)$ &(8, 768)& (30, 569)& (14, 690)& (3, 2201)& (33, 351)& (2, 5300) \\ \hline \hline
Dataset & Ringnorm & Heart & Sonar & {Thyroid} & Liver & German \\\hline
 $(p, N)$& (20, 7400)& (13, 270)&  (60, 208) & (20, 7200)&  (6, 345)& (20, 1000)\\ \hline
    \end{tabular}
\caption{Description of binary classification datasets used in Sec.~\ref{sec:two_layer}. The pair $(p, N)$ indicates the (dimension, training size).}
\label{tab:datasets_2class}
\end{center}
\end{table}

\begin{table}[ht!]
\begin{center}
    \begin{tabular}{ | c | c | c |  c | c |}
    \hline
    Dataset & CLASMK-ML  &2L-CLASMKL-ML & MKL & 2L-MKL \\ \hline
Diabetes &$70.8\pm3.4$ & $75.0\pm2.2$&$75.8\pm 2.5$ & ${\bf 76.6\pm 1.9}$ \\\hline
     {Breast} & $97.0\pm 0.7 $& ${\bf 97.1\pm 1.0}$& $96.5\pm 0.8$& $96.9\pm 0.7$    \\ \hline
    Australian   & $87.2\pm 5.7$  & ${\bf 87.8\pm5.5}$   & $85\pm 1.5$&  $85.7\pm 1.6$\\\hline
{ Titanic} & $78.5\pm 0.8$ & ${\bf 78.9\pm 0.7}$ & $77.1\pm 2.9$& $77.8\pm 2.6$\\\hline
{Ionosphere} &  $93.9\pm 1.58$ &${\bf 94.4\pm 2.0}$& $91.7 \pm 1.9$& ${\bf 94.4\pm 0.9}$ \\\hline
{Banana} & $90.2\pm 0.4$ & ${\bf  90.3\pm0.5}$ & $90.2\pm 2$& $90.2\pm 1.6$\\\hline
Ringnorm & $97.9\pm 0.4$& ${\bf 98.5\pm 0.4}$& $98.1\pm 0.8$  & ${\bf 98.5\pm 0.8}$\\\hline
Heart & $79.7\pm3.1$ & $81.1\pm 2.7$ & $83.0\pm 2.9$& ${\bf 83.6\pm 2.4}$ \\\hline
{Sonar} & $79.9\pm 4.3$ & ${\bf 84.7\pm 3.2}$& $78.3\pm 3.5$&$84.6\pm 2.4$ \\\hline
{Thyroid} & $94.4\pm 0.3$& $94.5\pm 0.3$& $92.9\pm 2.9$& ${\bf 94.8\pm 2.2}$ \\\hline
Liver & $63.8\pm 4.4$ & ${\bf 64.5\pm 3.1}$ & $62.3\pm 4.5$ & $62.7\pm3.1$\\\hline
German &$72.1\pm 2.4$ & ${\bf 74.2\pm 1.8}$ & $71.4\pm 2.8$& ${\bf 74.2\pm 2.0}$\\\hline
Summary &$83.8\pm 2.2$& ${\bf 85.1\pm2.0}$& $83.5 \pm2.4$ & ${\bf 85.0 \pm1.9}$\\\hline
    \end{tabular}
\caption{Experiments conducted on benchmark UCI datasets and comparison to other Multiple Kernel and hierarchical Multiple Kernel techniques.}
\label{tab:accuracy}
\end{center}
\end{table}

\subsection{Multi-Layer CLASMK-ML Performance with Respect to the Number of Layers}
\label{sec:multi_layer}
The datasets used in this section are: optdigits, pendigits, mnist, banana, phoneme,
adult. The dimensionality, data size and number of classes are summarized in Table.~\ref{tab:datasets2}. The setup is the same as in Sec.~\ref{sec:two_layer}. By default, we use 17 kernel functions: polynomial kernel with degree $d\in\{2^0, 2^1, \cdots, 2^6\}$ and RBF kernel with $\sigma\in \{2^{-3}, 2^{-2},\cdots,2^{6}\}$.

\begin{table}[ht!]
\begin{center}
    \begin{tabular}{ | c | c | c |  c |}
    \hline
    Dataset   & Optdigits & Pendigits &    MNIST  \\\hline
$(C,p,N)$& (10, 64, 1797) &  (10, 16, 3498)& (10, 784, 60000+10000) \\ \hline \hline
Dataset &Banana & Adult & Phoneme \\\hline
$(C,p,N)$& (2, 2, 5300)& (2, 14, 32561+16281)& (2, 5, 5404)\\ \hline
    \end{tabular}
\caption{Description of datasets used in Sec.~\ref{sec:multi_layer}. The triplet $(C,p,N)$ indicates the (number of classes, dimension, training size). When a default testing set is available, we write (number of classes, dimension, training size+testing size).}
\label{tab:datasets2}
\end{center}
\end{table}

\subsubsection{Classification Accuracy vs Number of Kernel Functions}

In this section, we investigate the empirical effect of the increasing number of kernel functions on the classification accuracy, where we keep the range of the kernel parameters unchanged. As shown in Fig.~\ref{fig:banana_acc_v_layer}, the best result is achieved when we use 98 kernel functions. Generally speaking, a decreasing number of kernels results in a degraded performance. Similar results can be observed in Fig.~\ref{fig:pendigits_acc_kernel}, Fig.~\ref{fig:phoneme_acc_kernels} and Fig.~\ref{fig:optdigits_acc_kernel}.
\begin{figure}[p!]
\centering
\includegraphics[width=90mm]{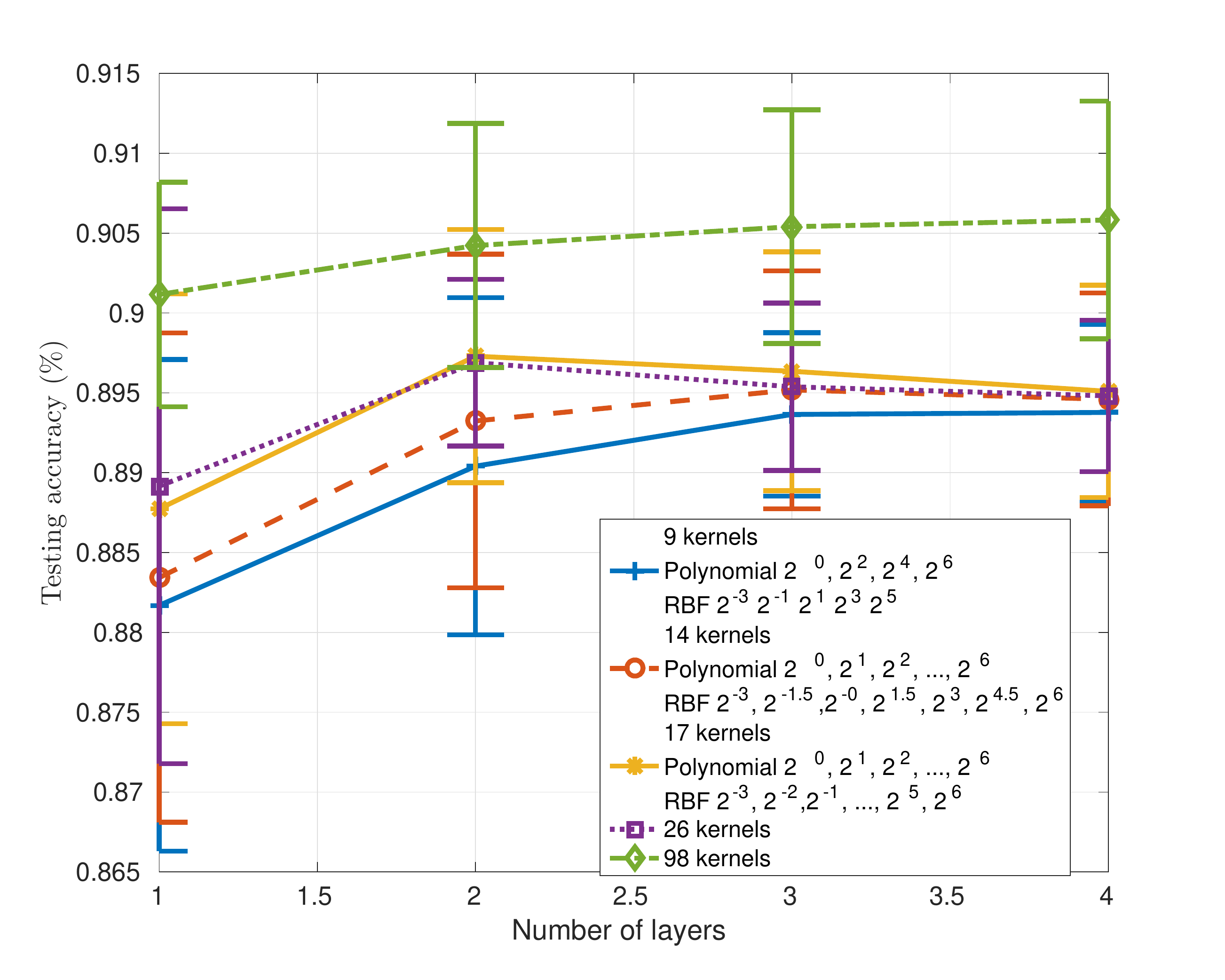}
\caption{Classification accuracy on the dataset banana using different number of kernel functions.}
\label{fig:banana_acc_v_layer}
\end{figure}

\begin{figure}[ht!]
\centering
\includegraphics[width=90mm]{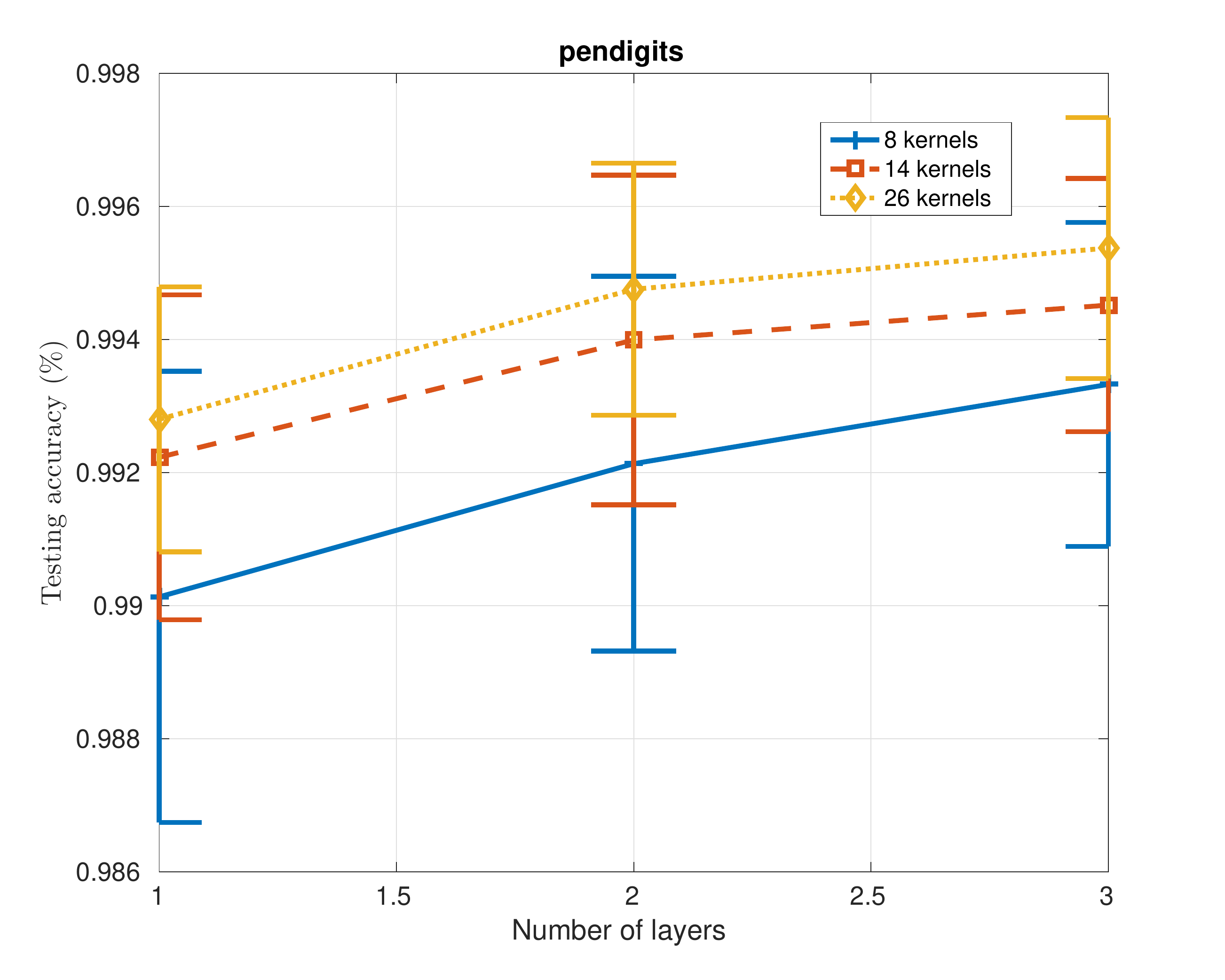}
\caption{Classification accuracy on the dataset pendigits using different number of kernel functions.}
\label{fig:pendigits_acc_kernel}
\end{figure}

\begin{figure}[ht!]
\centering
\includegraphics[width=90mm]{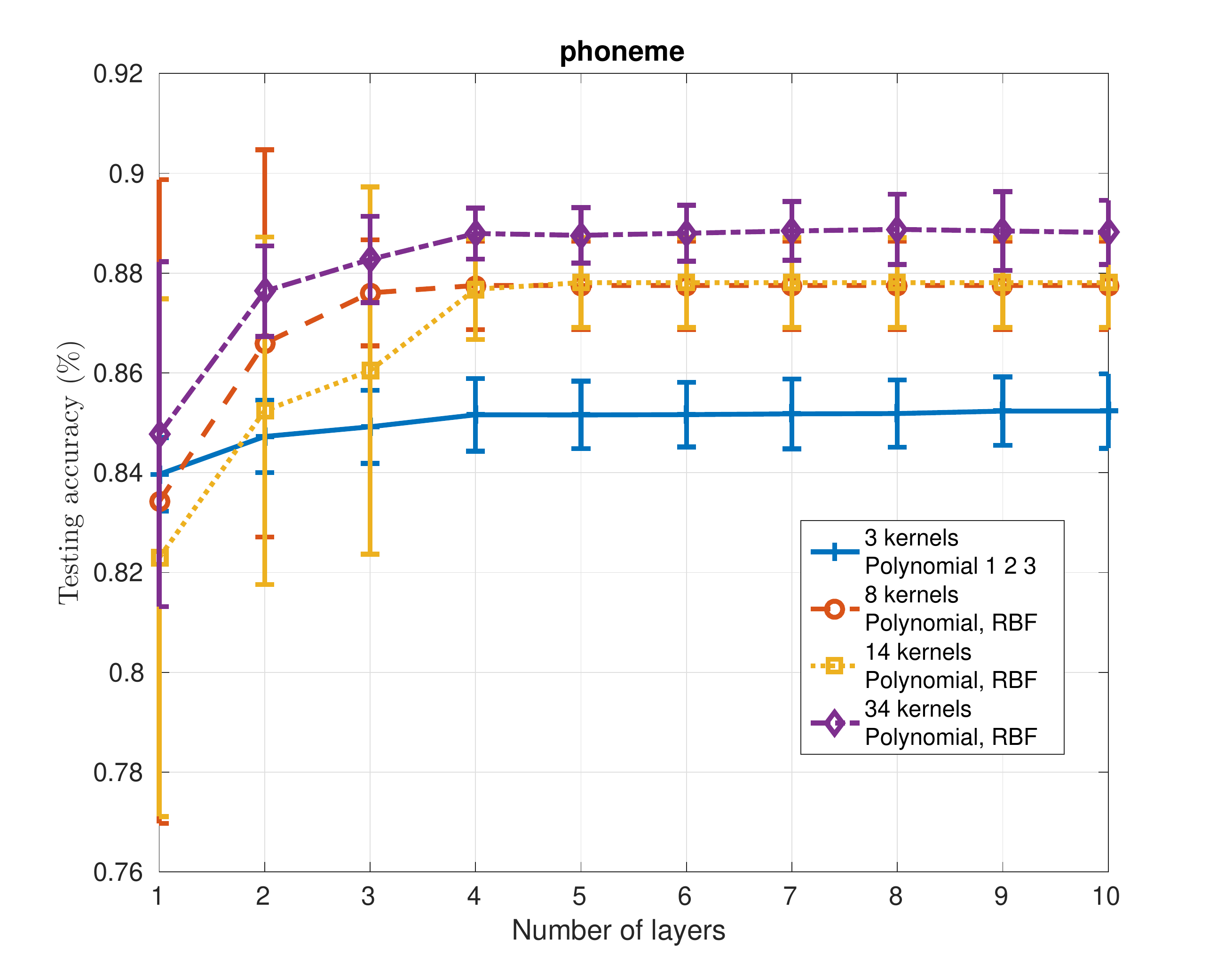}
\caption{Classification accuracy on the dataset phoneme using different number of kernel functions.}
\label{fig:phoneme_acc_kernels}
\end{figure}

\begin{figure}[ht!]
\centering
\includegraphics[width=90mm]{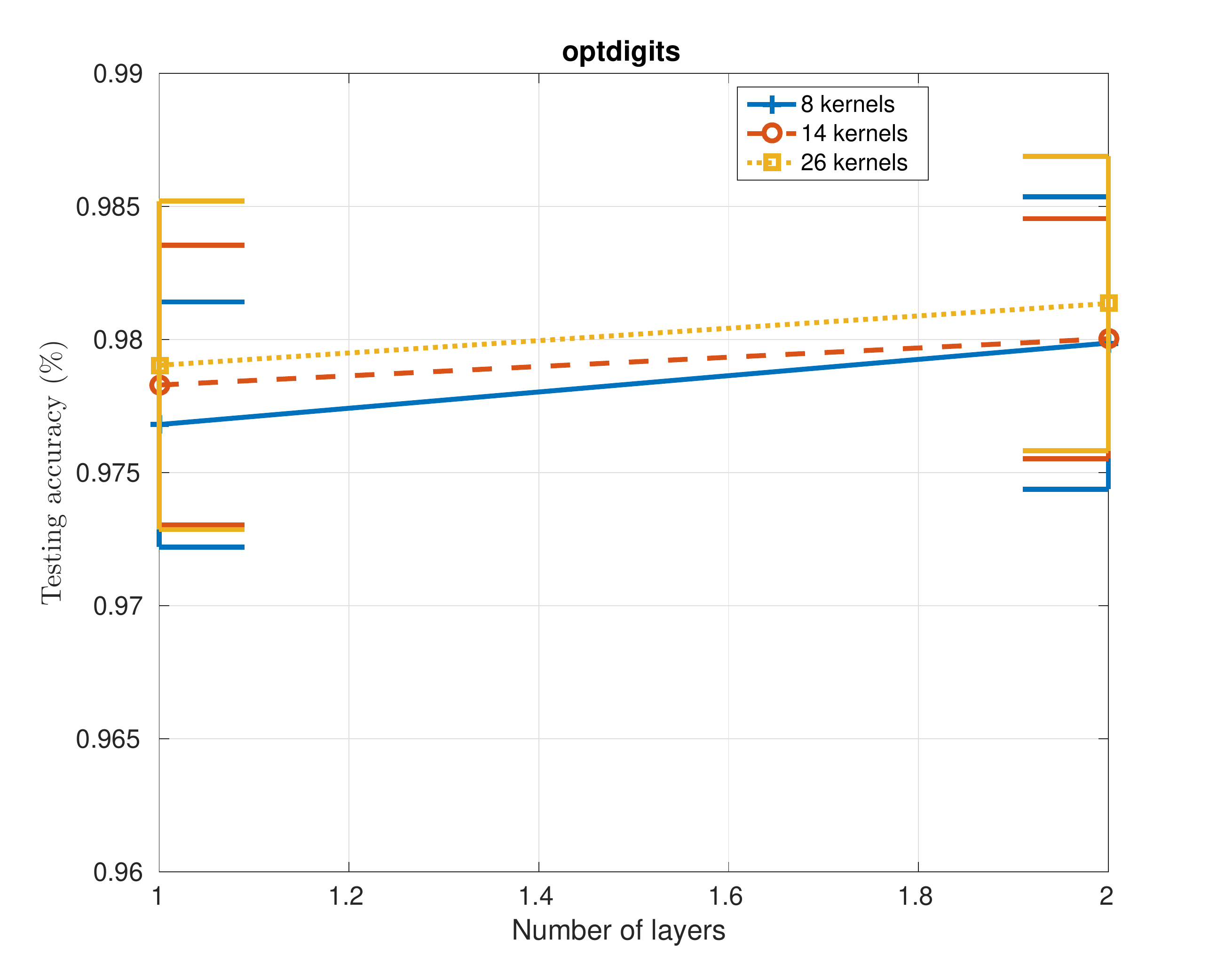}
\caption{Classification accuracy on the dataset optdigits using different number of kernel functions.}
\label{fig:optdigits_acc_kernel}
\end{figure}

\subsubsection{Classification Accuracy vs Number of Layers with Different Training Sizes}

To study how the classification accuracy changes with respect to the number of layers for different training sizes, we conduct experiments using randomly selected subsets for training. Results are shown in Fig.~\ref{fig:mnist_acc_size} to Fig.~\ref{fig:pendigits_acc_size}. We can see that the accuracy increases with increasing number of training sizes as expected.

\begin{figure}[ht!]
\centering
\includegraphics[width=90mm]{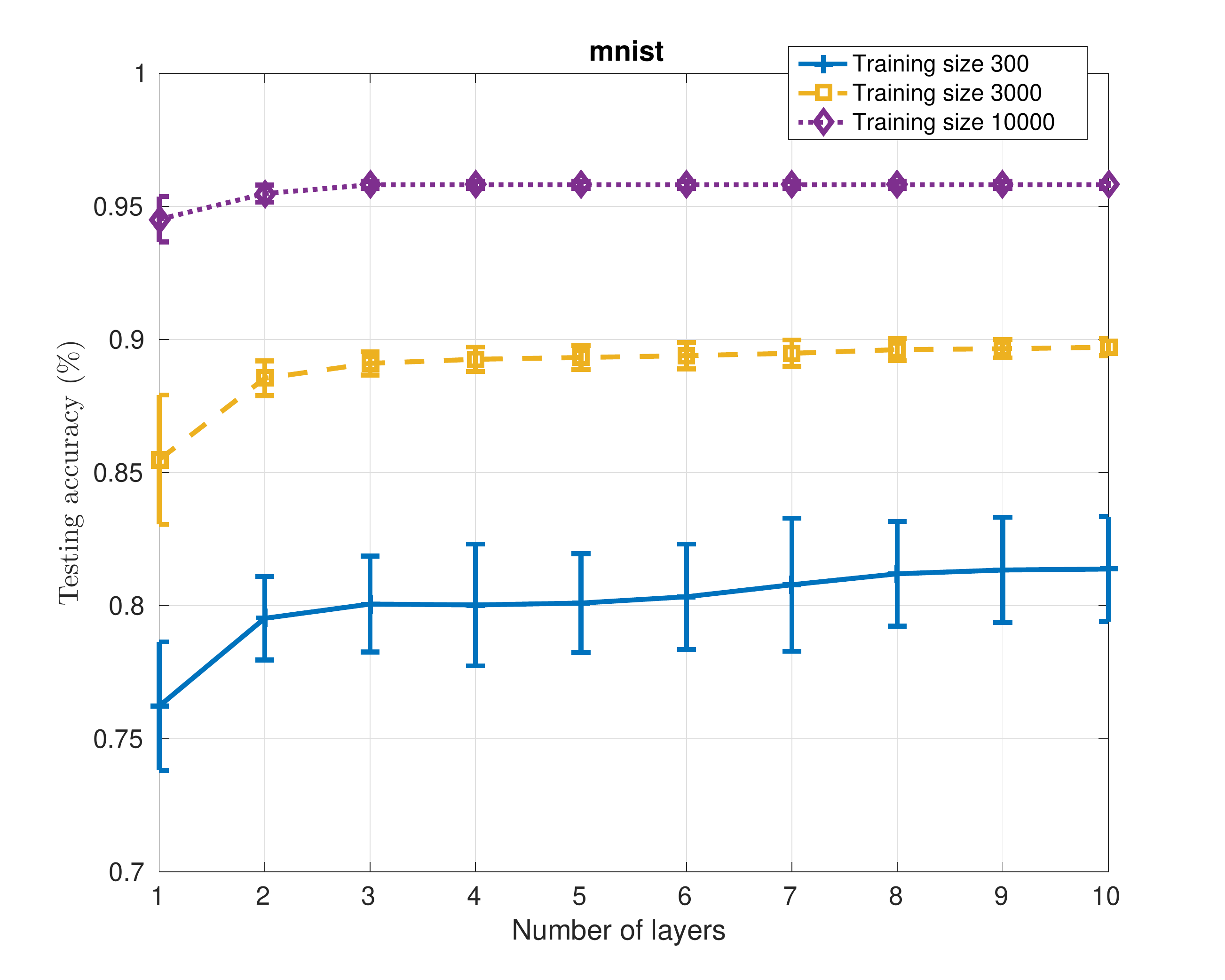}
\caption{Classification accuracy on the dataset MNIST using different training sizes.}
\label{fig:mnist_acc_size}
\end{figure}
\begin{figure}[ht!]
\centering
\includegraphics[width=90mm]{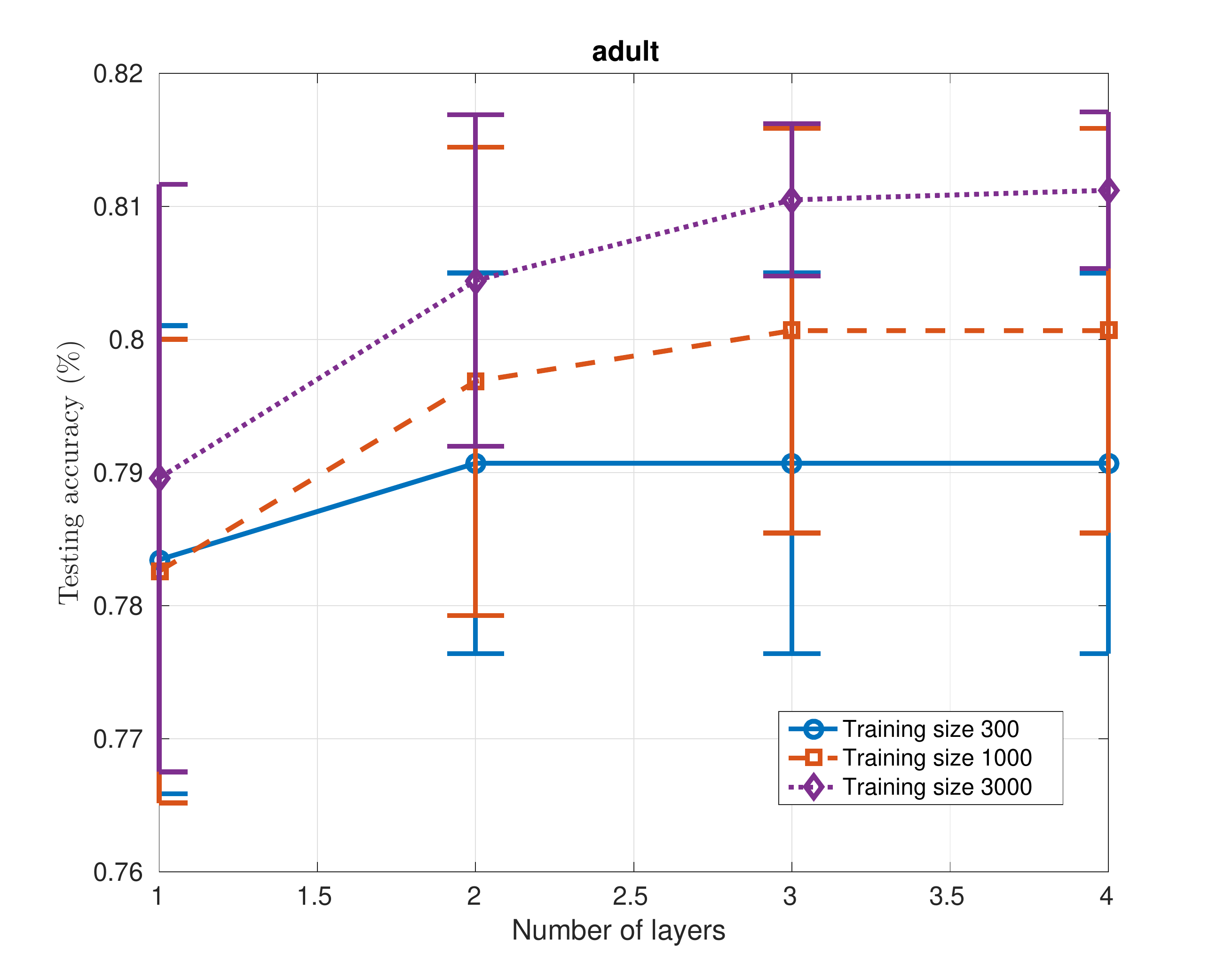}
\caption{Classification accuracy on the dataset adult using different training sizes.}
\label{fig:adult_acc_size}
\end{figure}
\begin{figure}[ht!]
\centering
\includegraphics[width=90mm]{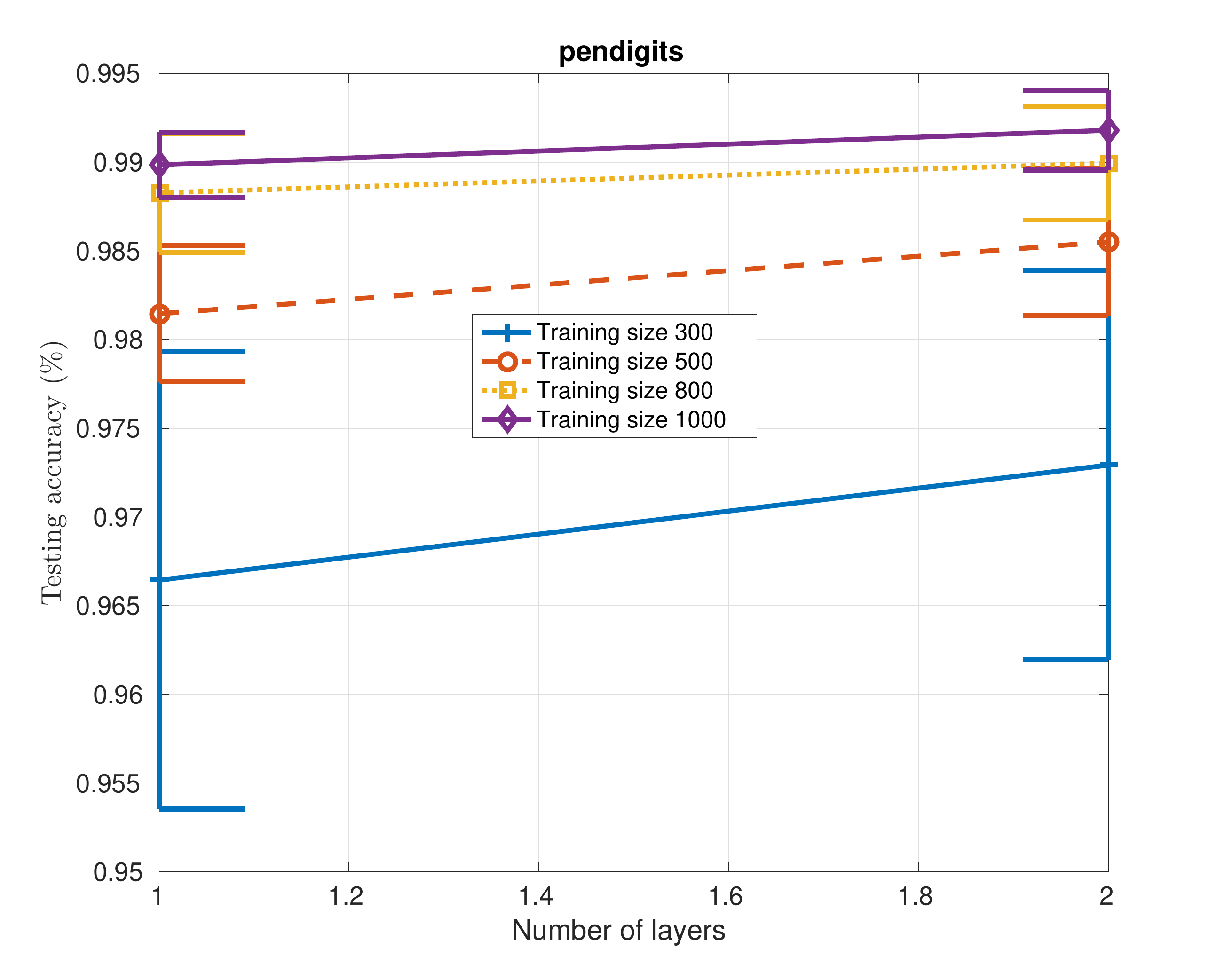}
\caption{Classification accuracy on the dataset pendigits using different training sizes.}
\label{fig:pendigits_acc_size}
\end{figure}

\subsubsection{Training Time}

The training time is evaluated using 8 core Intel i7 CPU with 16G of RAM running MATLAB (2015b) on Linux 4.4.0. The time reported is the optimization step for finding the optimal $\bss{\nu}$ on each layer (note that it is not the accumulated time of all previous layers). The results can be found in Fig.~\ref{fig:banana_time_v_layer} to Fig.~\ref{fig:pendigits_time_size}. More specifically, Fig.~\ref{fig:banana_time_v_layer} to Fig.~\ref{fig:optdigits_time_kernel} have shown the training time with different number of kernels and Fig.~\ref{fig:adult_time_size} to Fig.~\ref{fig:pendigits_time_size} have illustrated the training time for various training sizes.
From the numerical results, we observe that the training time for the optimization is in linear with respect the number of kernel functions.
\begin{figure}[ht!]
\centering
\includegraphics[width=90mm]{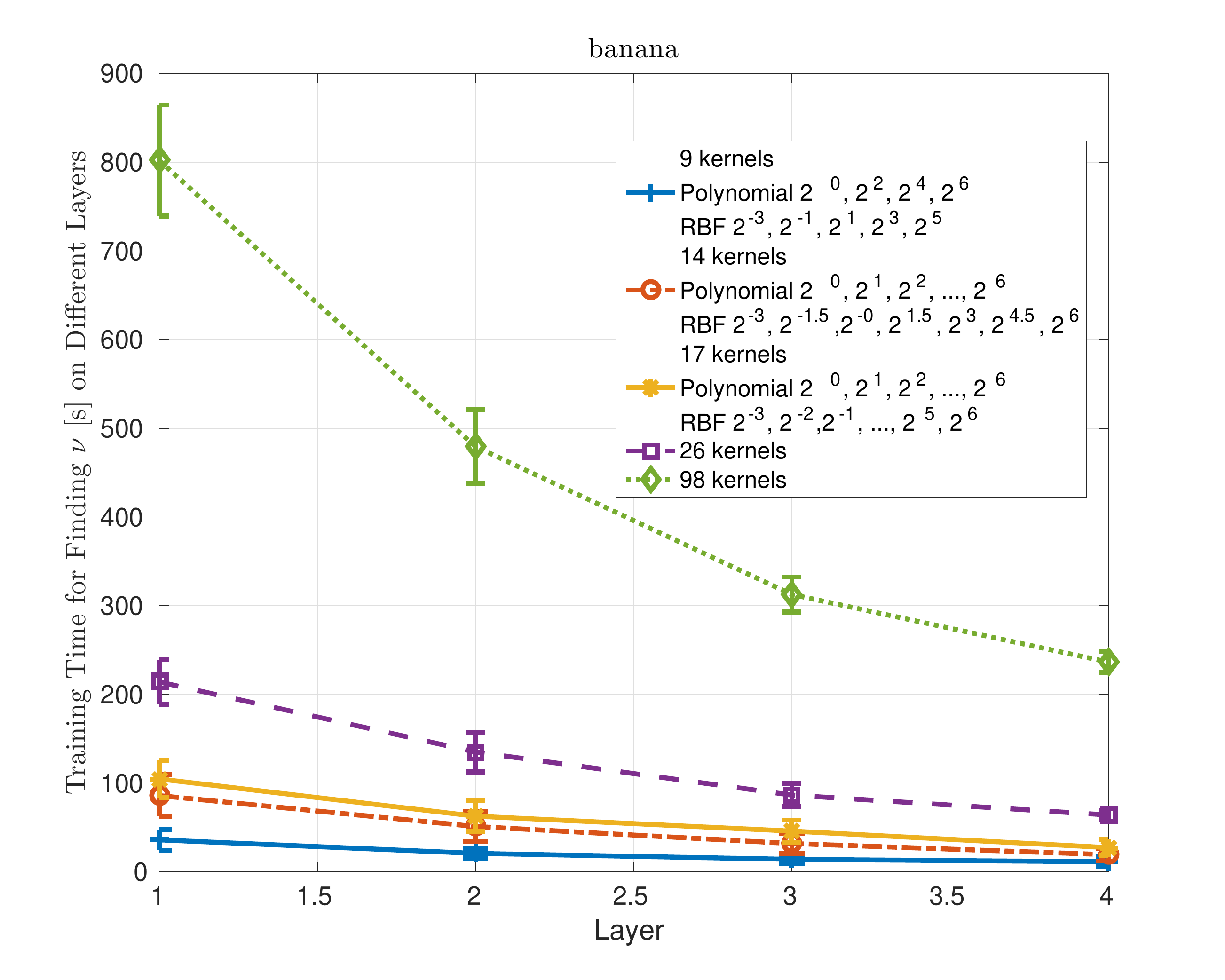}
\caption{Training time on the dataset banana with different numbers of kernel functions.}
\label{fig:banana_time_v_layer}
\end{figure}

\begin{figure}[ht!]
\centering
\includegraphics[width=90mm]{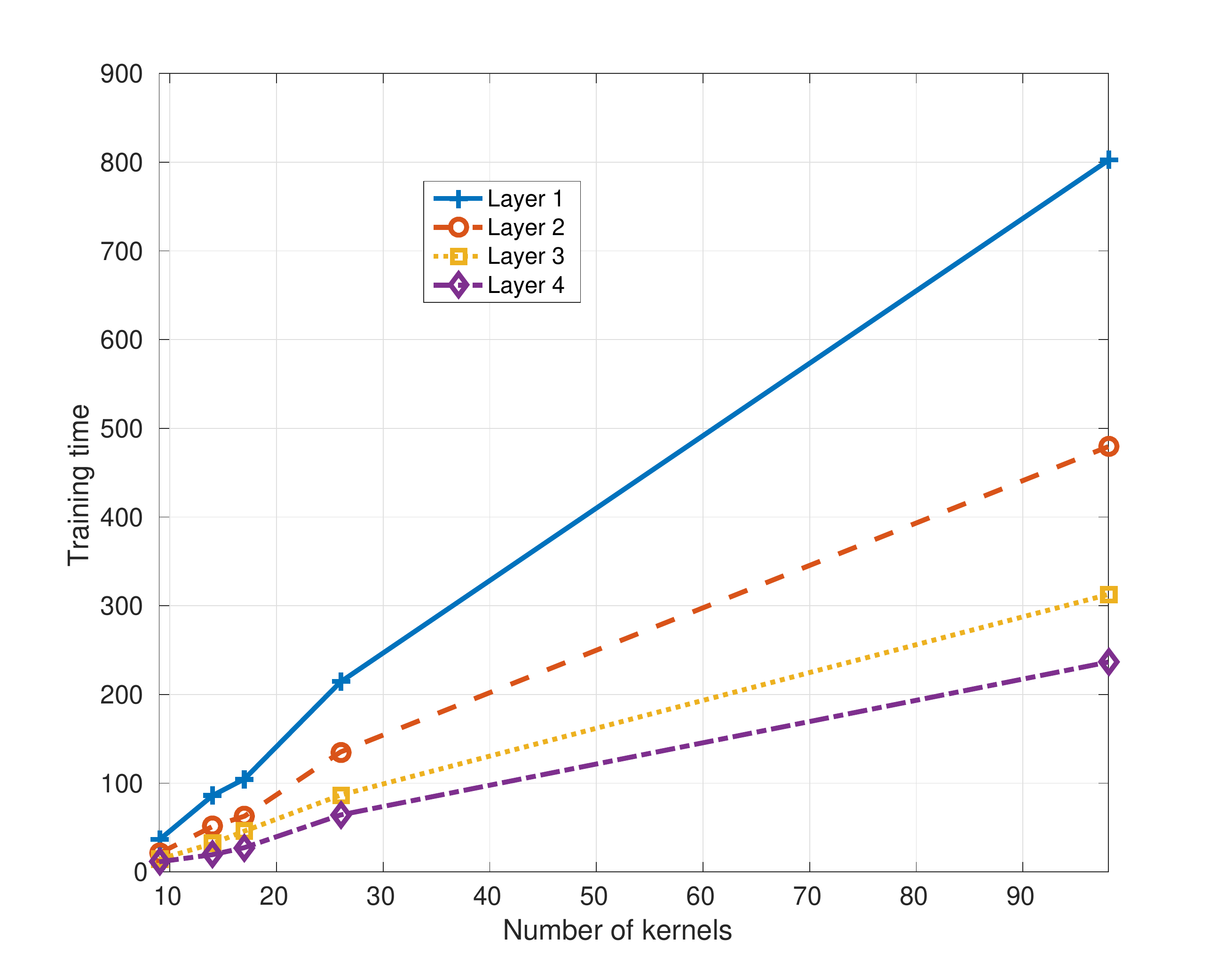}
\caption{Training time on the dataset banana with different numbers of kernel functions for different layers.}
\label{fig:banana_all_time}
\end{figure}
\begin{figure}[ht!]
\centering
\includegraphics[width=90mm]{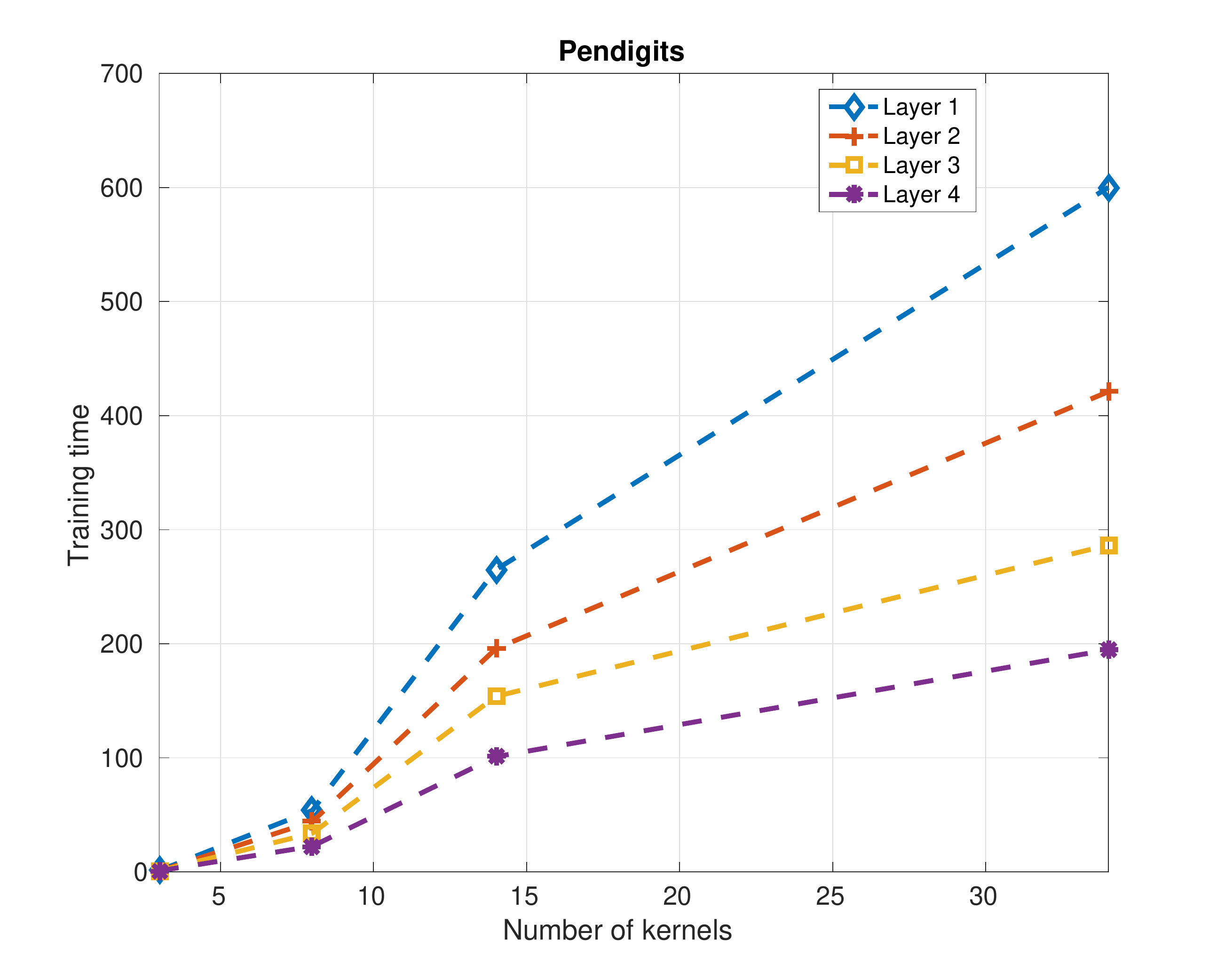}
\caption{Training time on the dataset phoneme with different numbers of kernel functions for different layers.}
\label{fig:phoneme_all_time}
\end{figure}
\begin{figure}[ht!]
\centering
\includegraphics[width=90mm]{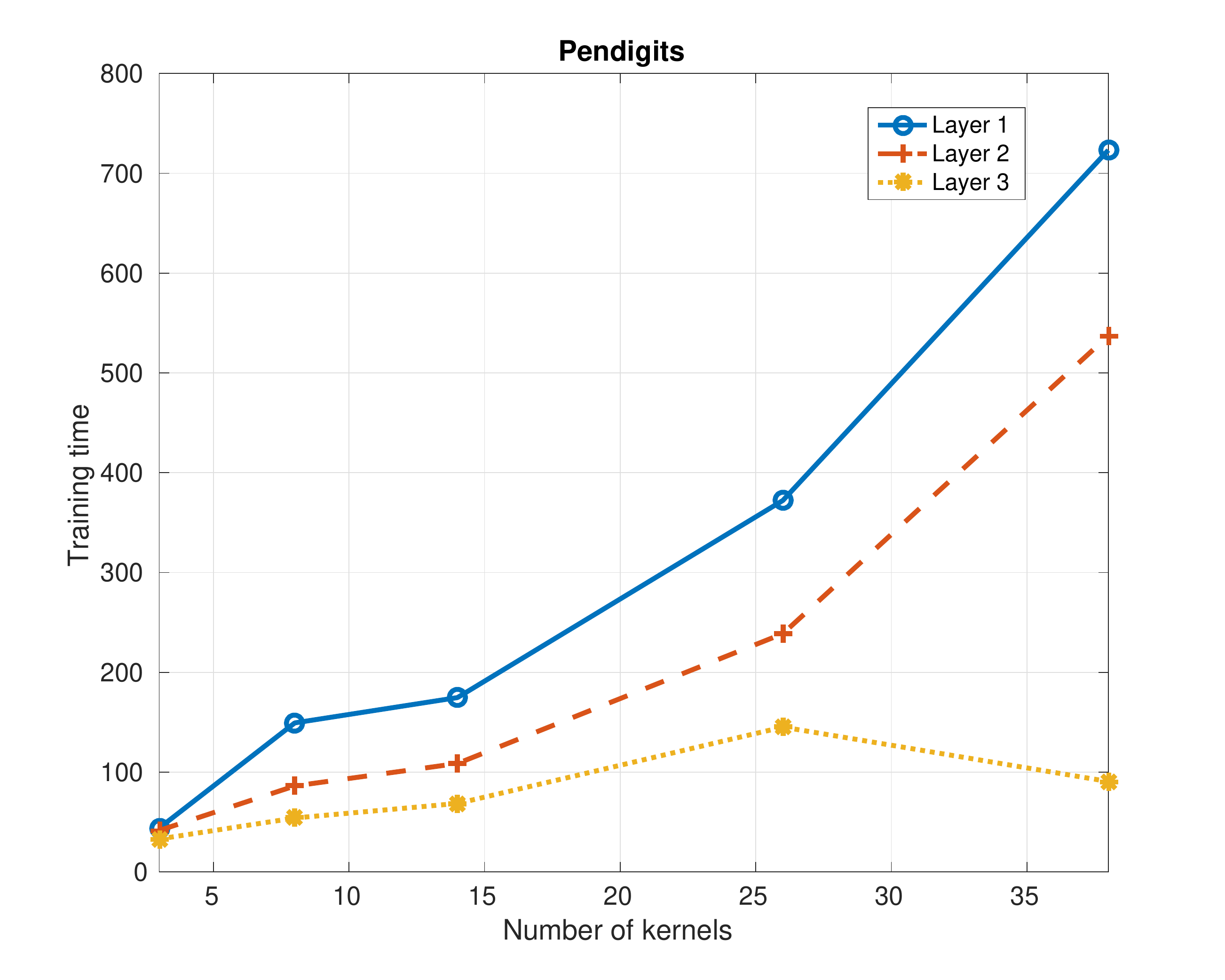}
\caption{Training time on the dataset pendigits with different numbers of kernel functions for different layers.}
\label{fig:pendigits_all_time}
\end{figure}

\begin{figure}[ht!]
\centering
\includegraphics[width=90mm]{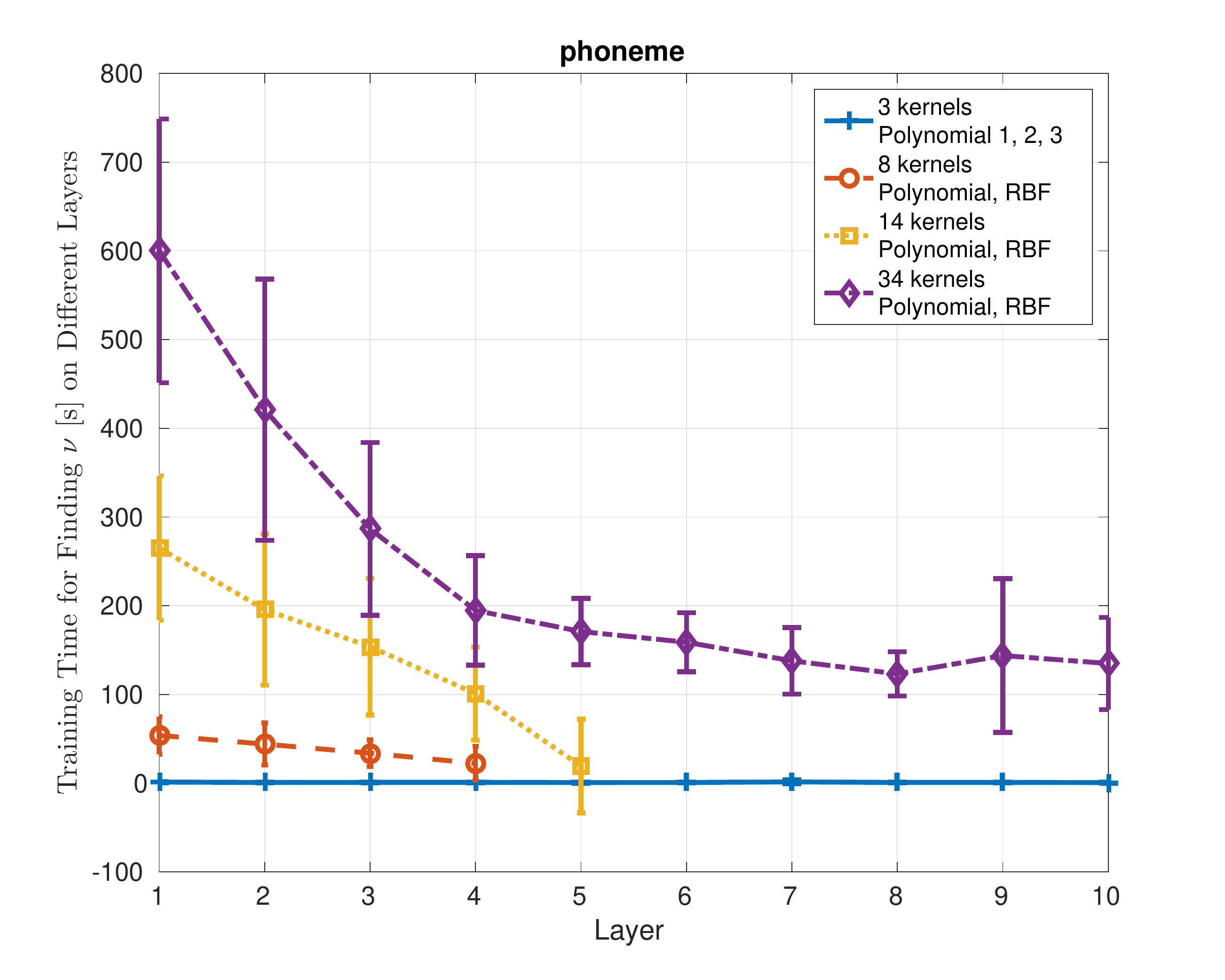}
\caption{Training time on the dataset phoneme with different numbers of kernels.}
\label{fig:phoneme_time_kernel}
\end{figure}

\begin{figure}[ht!]
\centering
\includegraphics[width=90mm]{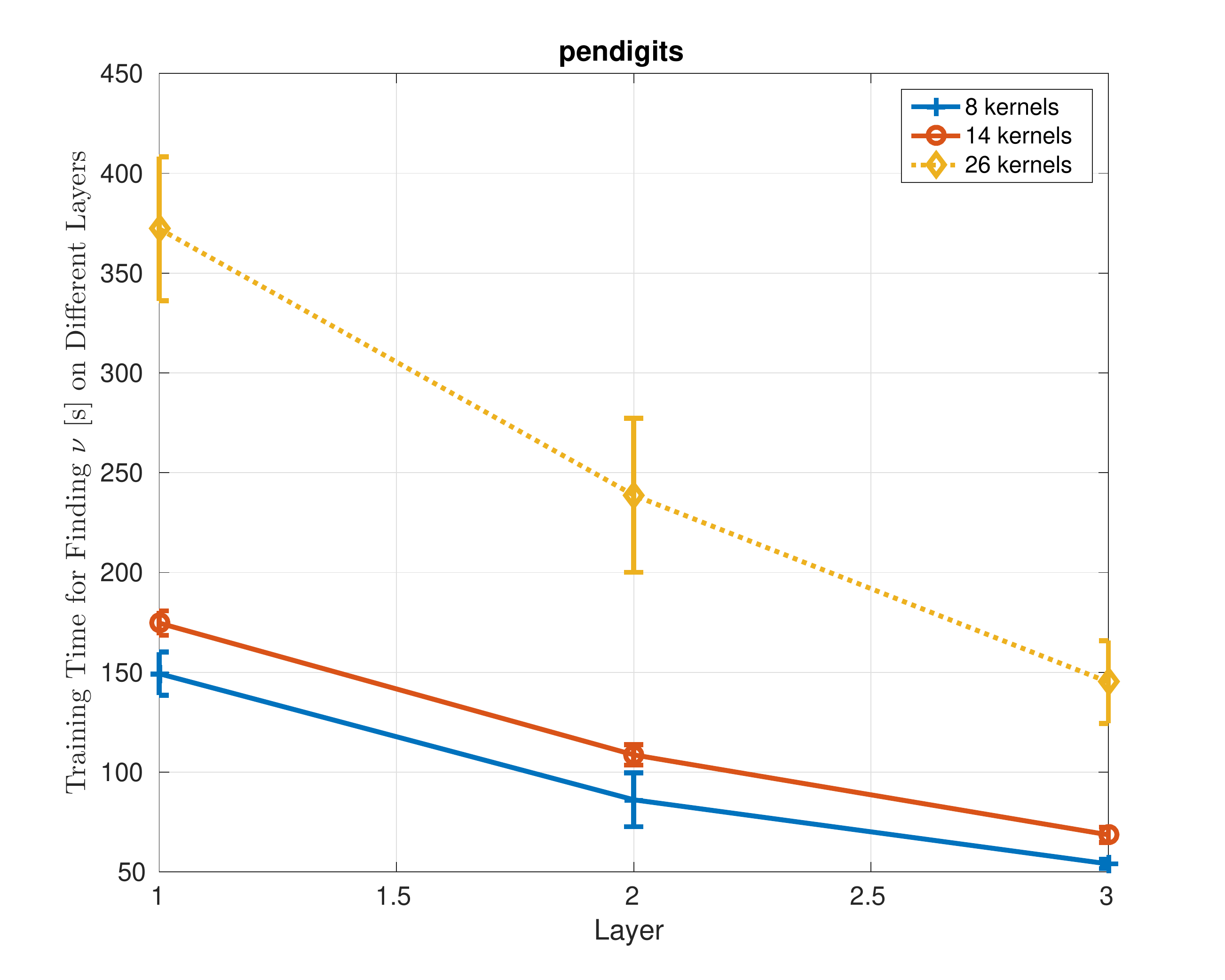}
\caption{Training time on the dataset pendigits with different numbers of kernels.}
\label{fig:pendigits_time_kernel}
\end{figure}

\begin{figure}[ht!]
\centering
\includegraphics[width=90mm]{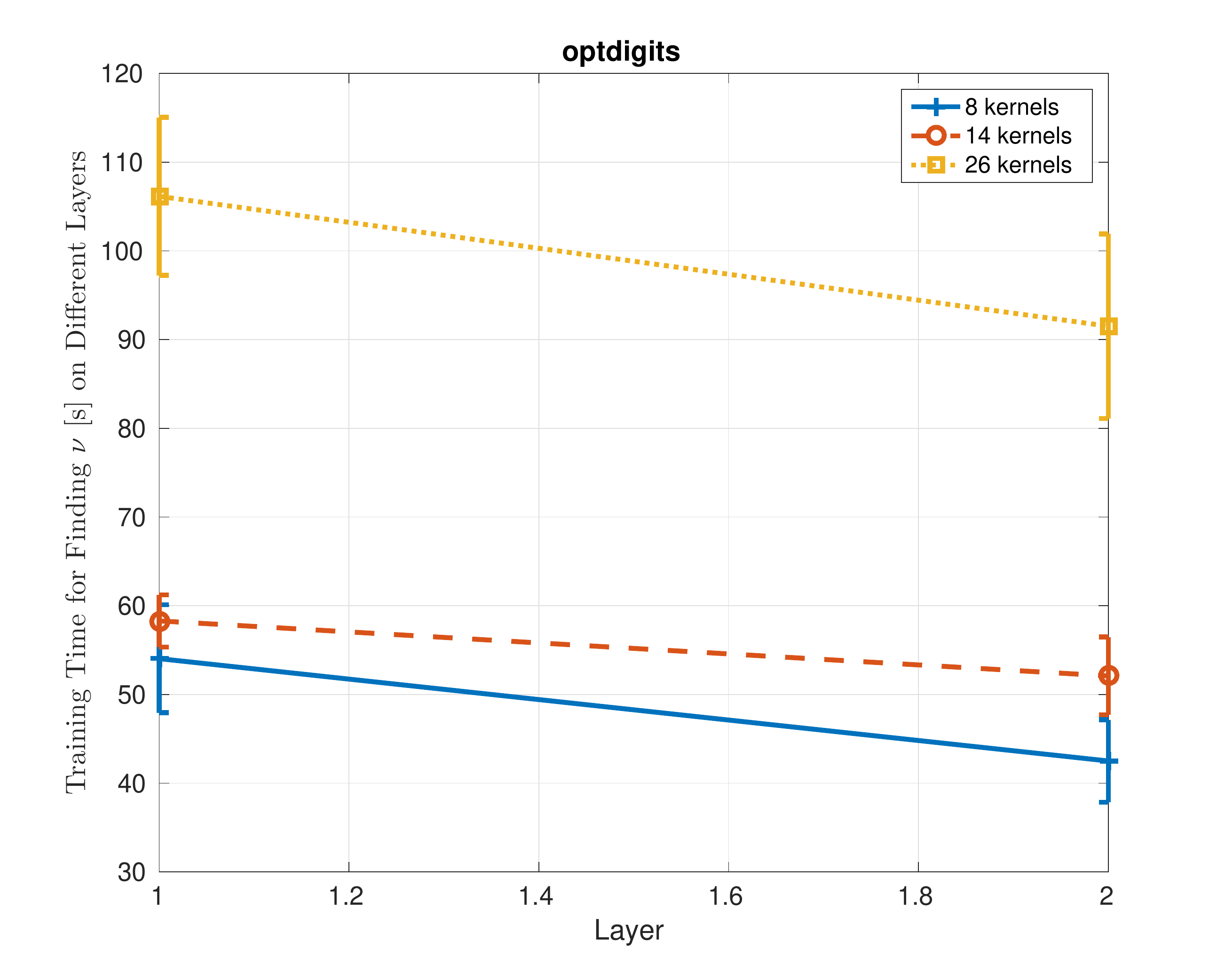}
\caption{Training time on the dataset optdigits with different numbers of kernels.}
\label{fig:optdigits_time_kernel}
\end{figure}

\begin{figure}[ht!]
\centering
\includegraphics[width=90mm]{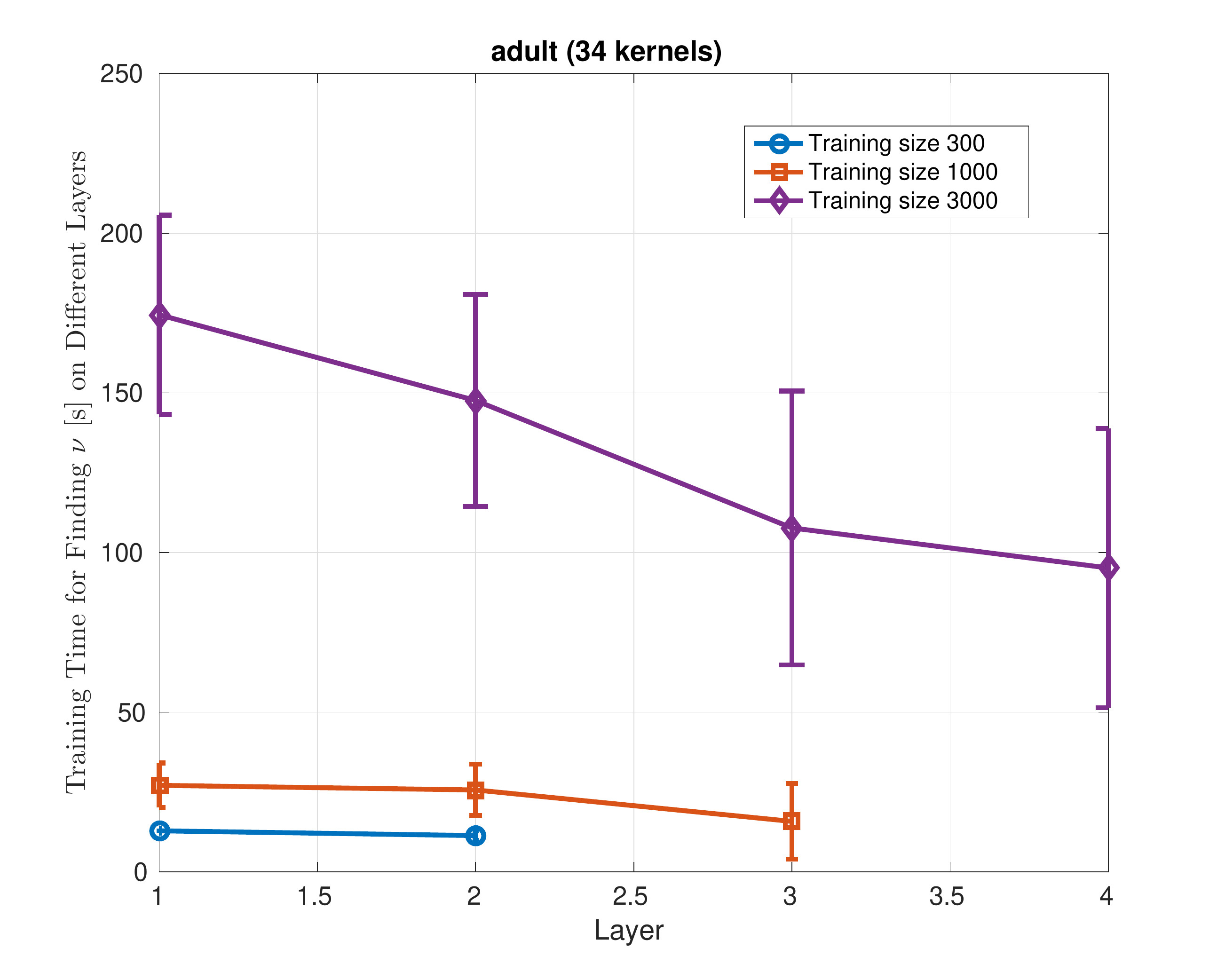}
\caption{Training time on the dataset adult with different training sizes.}
\label{fig:adult_time_size}
\end{figure}

\begin{figure}[ht!]
\centering
\includegraphics[width=90mm]{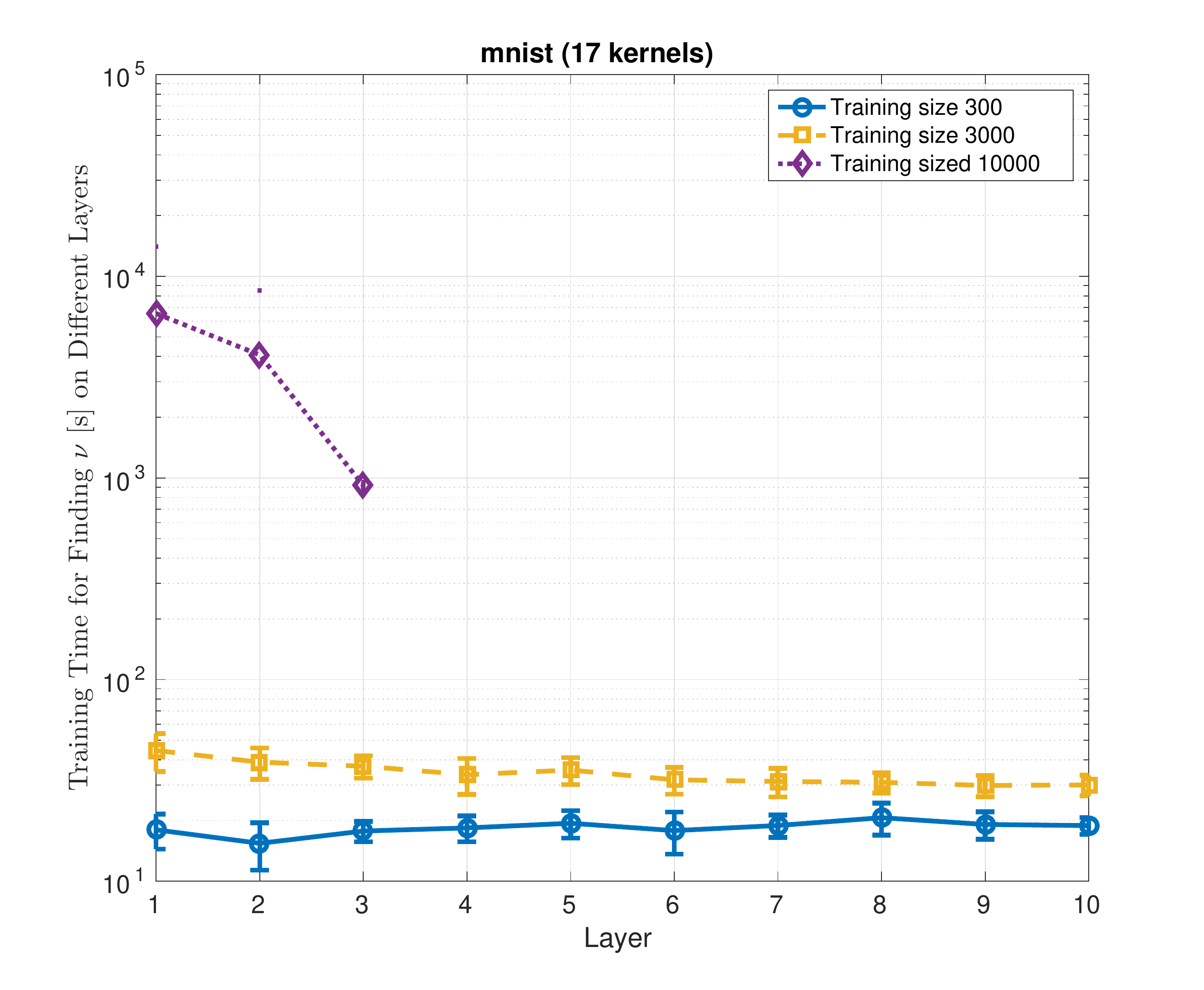}
\caption{Training time on the dataset mnist with different training sizes.}
\label{fig:mnist_time_size}
\end{figure}

\begin{figure}[ht!]
\centering
\includegraphics[width=90mm]{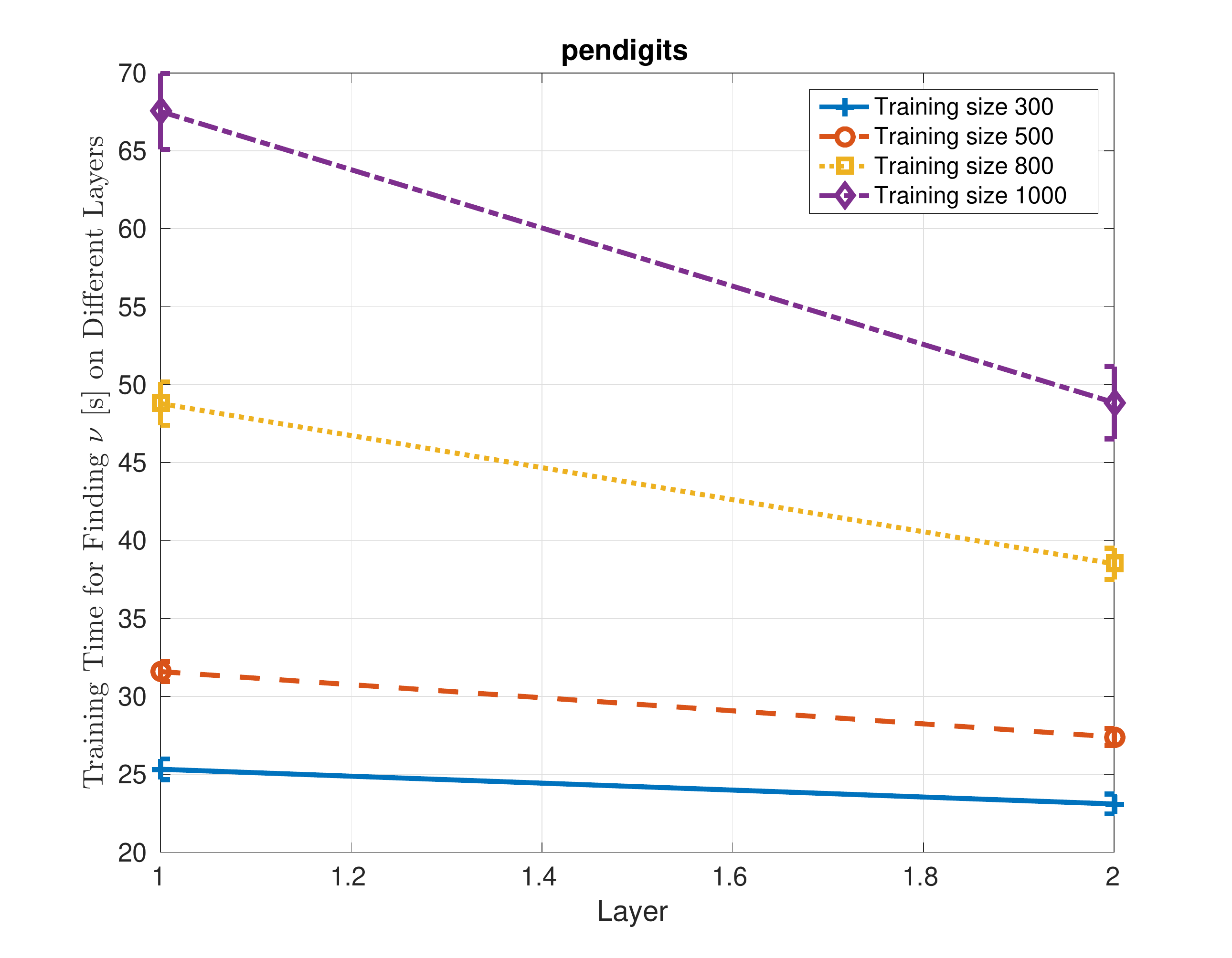}
\caption{Training time on the dataset pendigits with different training sizes.}
\label{fig:pendigits_time_size}
\end{figure}

\subsubsection{Dimension of Resulting Feature Space vs Number of Layers}
In this section, we show the resulting dimensionality of the feature space using the hierarchical CLASMK-ML learning technique. As shown in Fig.~\ref{fig:banana_dim_v_layer} to Fig.~\ref{fig:adult_dim_size}, feature dimension mainly increases with respect to the number of layers in a linear fashion due to the feature augmentation scheme. This is not ideal for large scale learning models, which gives rise to the importance of feature pruning as a future direction.

\begin{figure}[ht!]
\centering
\includegraphics[width=90mm]{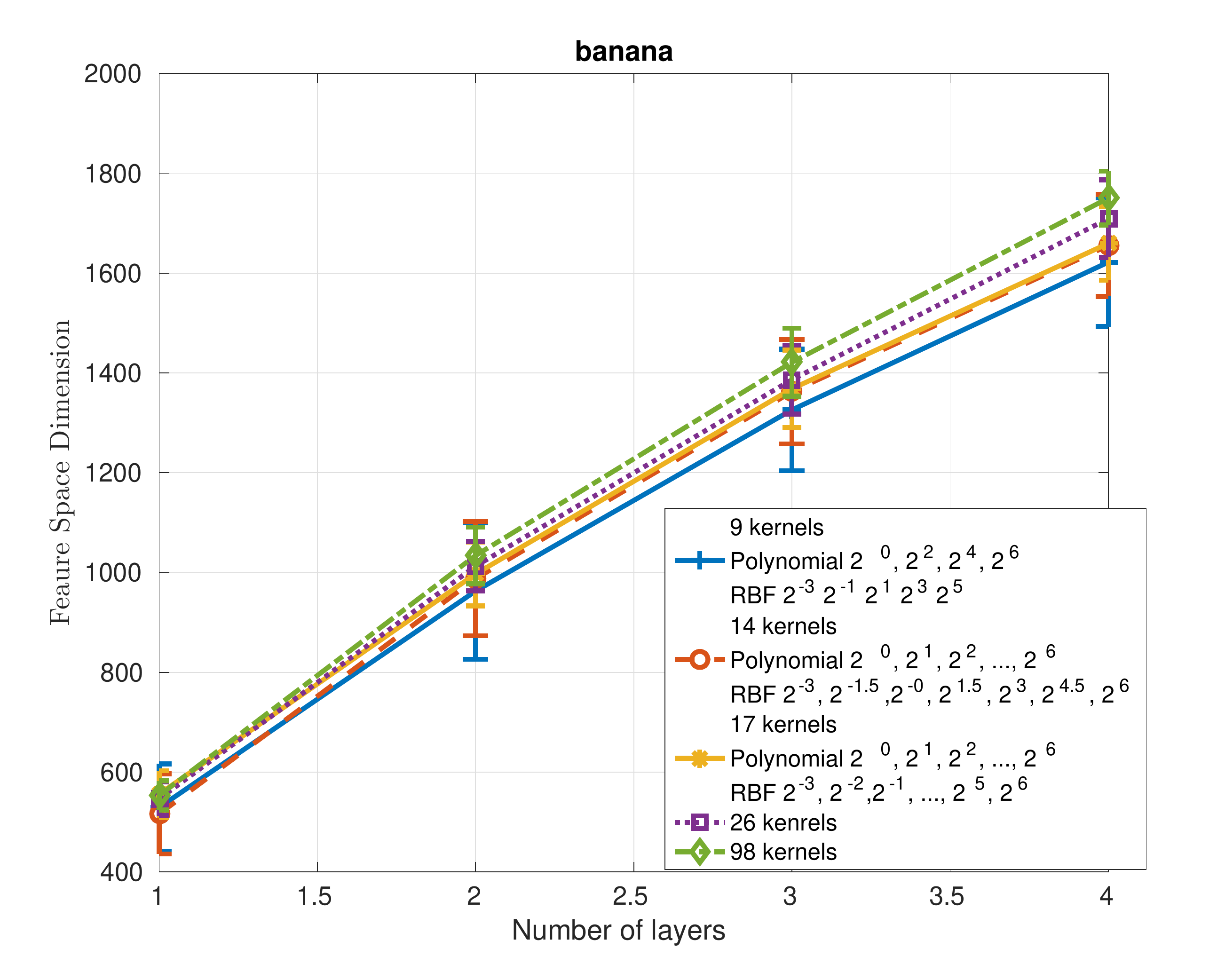}
\caption{Resulting feature dimension on the dataset banana with different numbers of kernel functions.}
\label{fig:banana_dim_v_layer}
\end{figure}
\begin{figure}[ht!]
\centering
\includegraphics[width=90mm]{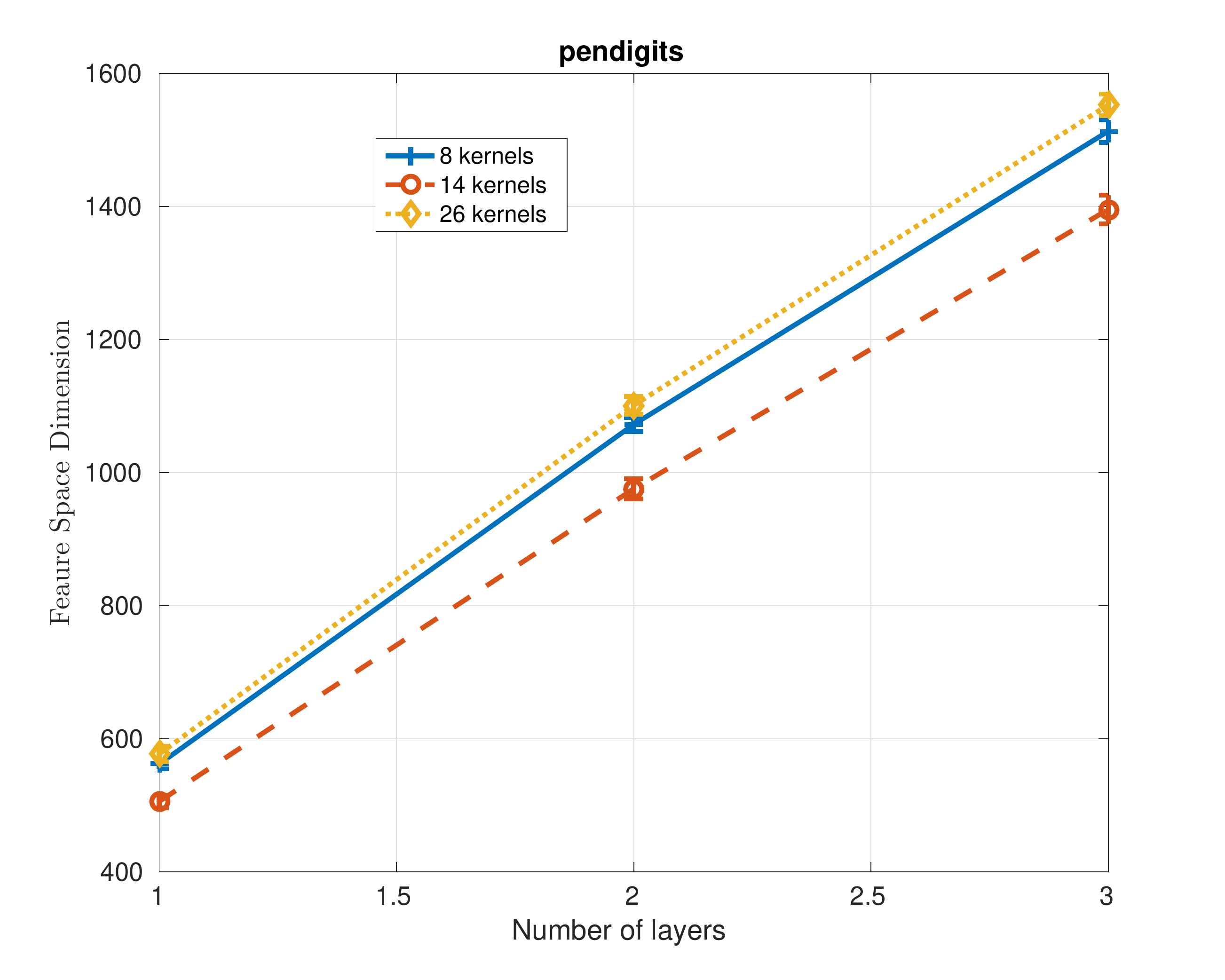}
\caption{Resulting feature dimension on the dataset pendigits with different numbers of kernel functions.}
\label{fig:pendigits_dim_kernel}
\end{figure}
\begin{figure}[ht!]
\centering
\includegraphics[width=90mm]{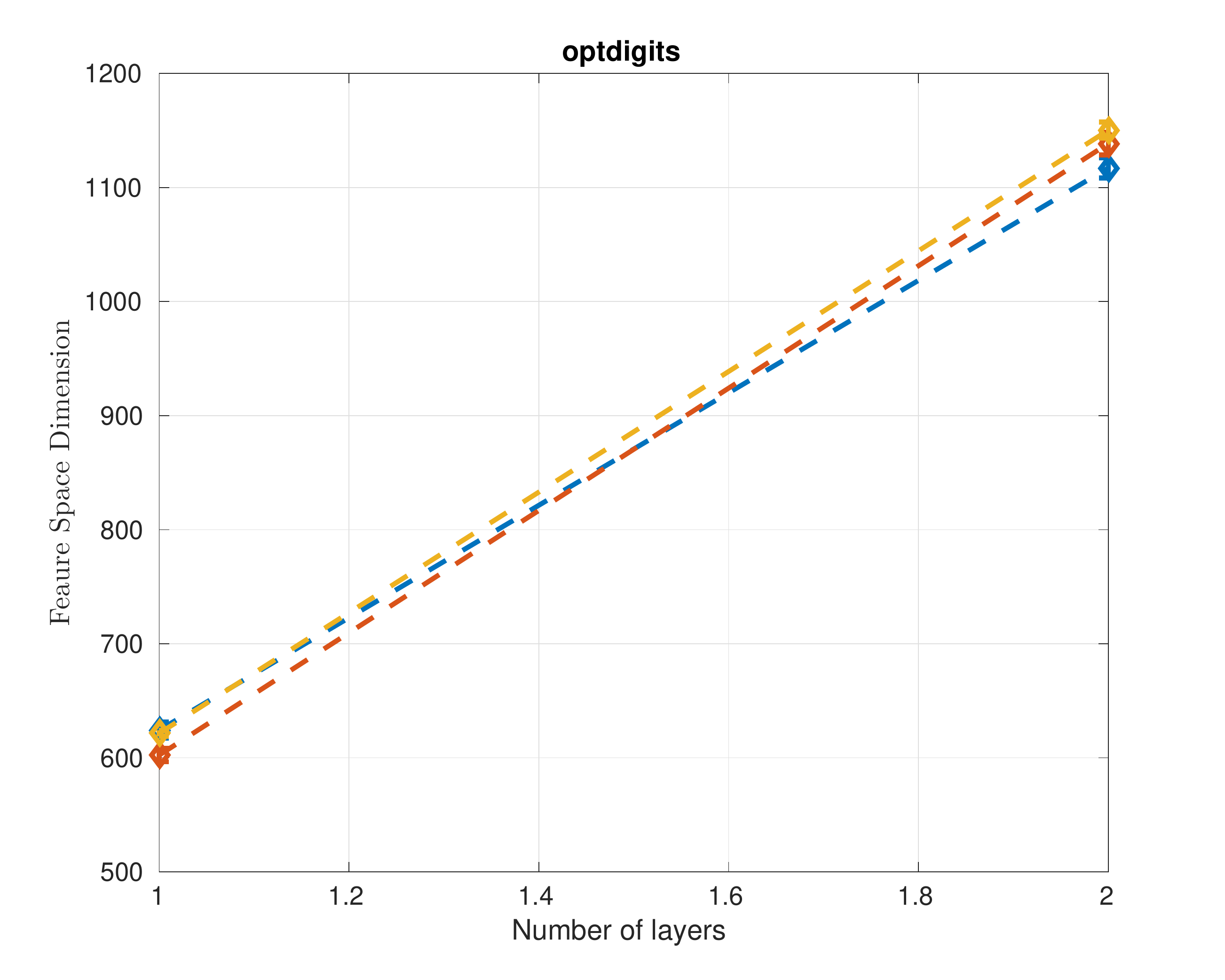}
\caption{Resulting feature dimension on the dataset optdigits with different numbers of kernel functions.}
\label{fig:optdigits_dim_kernel}
\end{figure}

\begin{figure}[ht!]
\centering
\includegraphics[width=90mm]{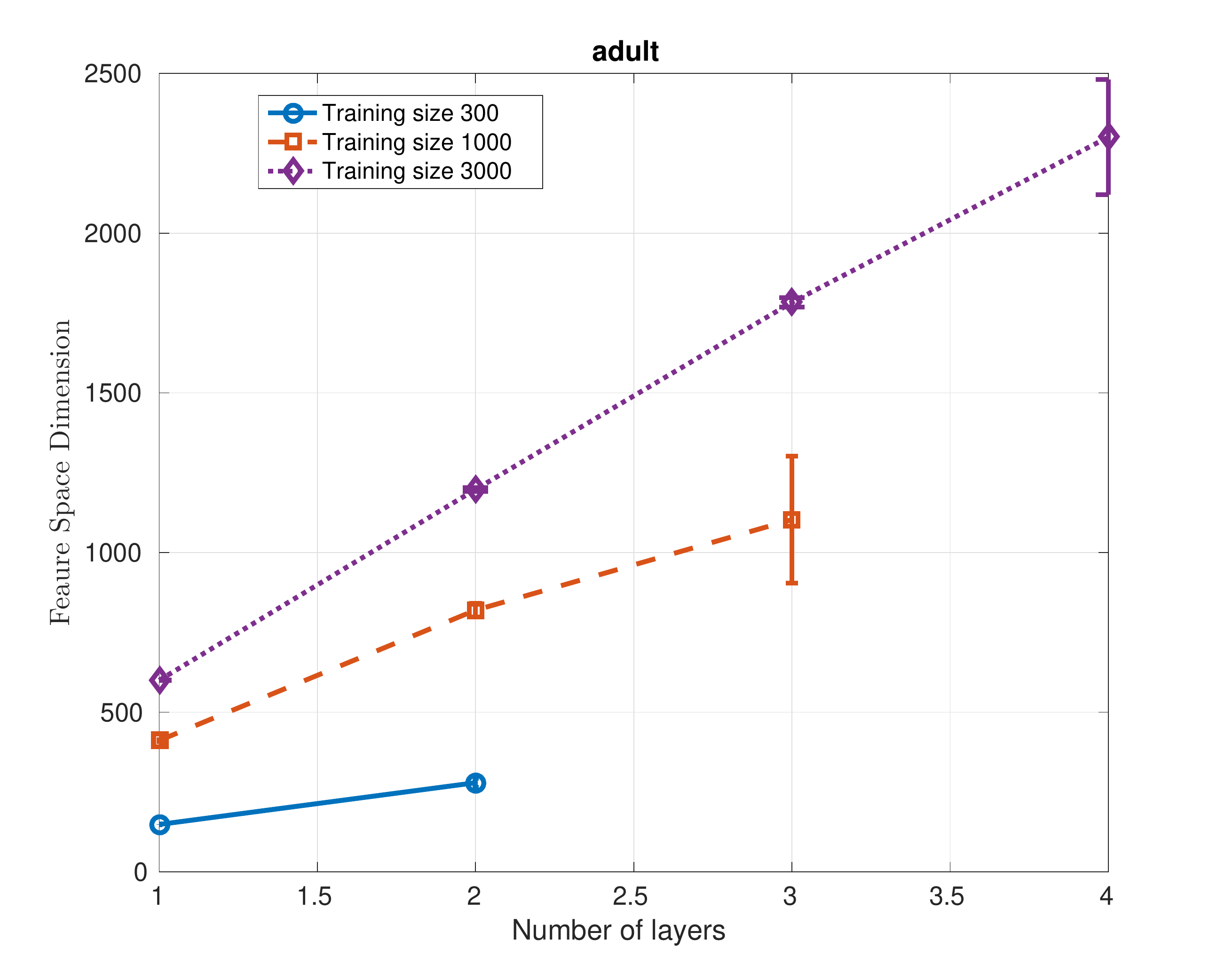}
\caption{Resulting feature dimension on the dataset adult with different training sizes.}
\label{fig:adult_dim_size}
\end{figure}
\begin{figure}[ht!]
\centering
\includegraphics[width=90mm]{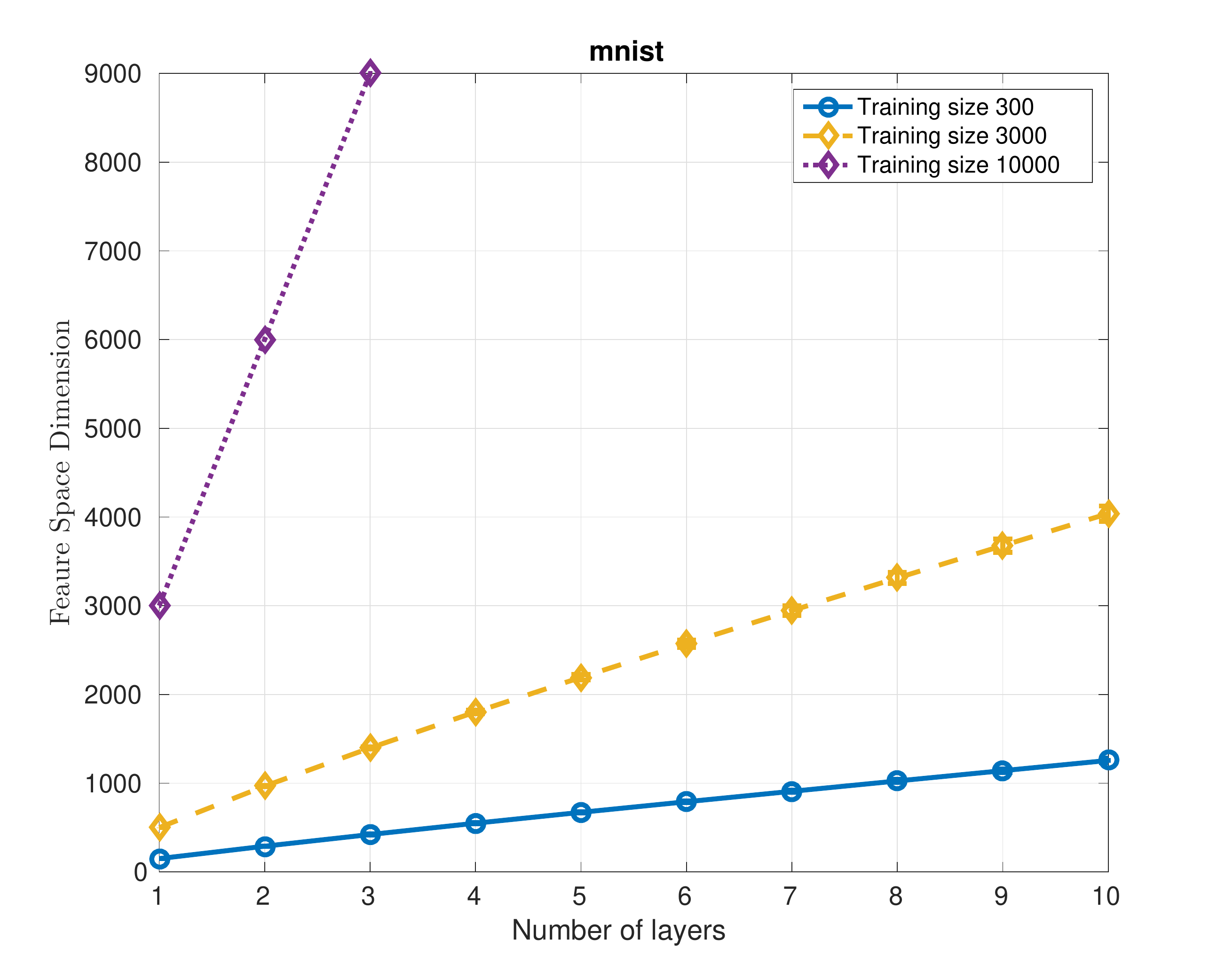}
\caption{Resulting feature dimension on the dataset mnist with different training sizes.}
\label{fig:adult_dim_size}
\end{figure}

\subsection{Visual Examples of the Estimated Weights}
We have shown some examples of the estimated weight matrix $\bss{\nu}\in\mathbb{R}^{C\times K}$ with $C$ classes and $K$ kernel functions in Fig.~\ref{fig:example_nu_banana}, Fig.~\ref{fig:example_nu_pendigits}, Fig.~\ref{fig:example_nu_optdigits} and Fig.~\ref{fig:example_nu_wdbc} for visual inspection.
The examples are shown as heat maps of the matrix $\bss{\nu}$, where a lighter color represents a higher value.
The datasets used are banana, pendigits, optdigits and wdbc. One can find the descriptions of the datasets in Table~\ref{tab:datasets_2class} and Table~\ref{tab:datasets2}. Despite the large variance on the dataset wdbc due to the small training size, we observe a fairly consistent estimation using random subsets for the other datasets.
\begin{figure}[ht!]
\centering
\includegraphics[width=90mm]{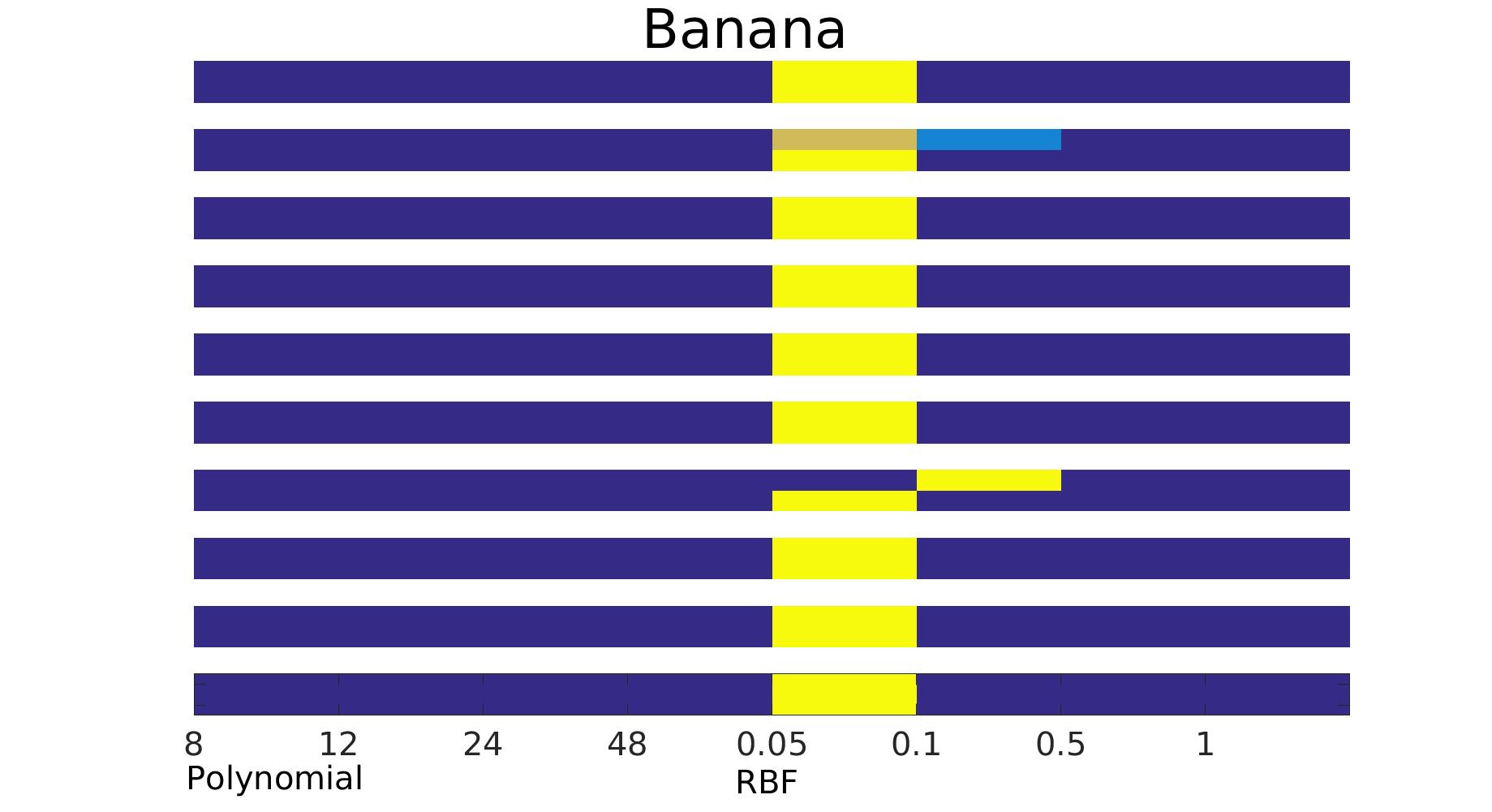}
\caption{In this figure, each row represents one example of the estimated weight matrix $\bss{\nu}\in\mathbb{R}^{C\times K}$ on the dataset banana, where $C$ is the number of classes and $K$ is the number of kernel functions. We repeat the experiment 10 times with randomized 3-fold testing and plot them in 10 different rows. The kernel functions are: the polynomial kernels with $d\in\{8,12,24,48\}$ and the RBF kernels with $\sigma\in\{0.05,0.1,0.5,1\}$.}
\label{fig:example_nu_banana}
\end{figure}

\begin{figure}[ht!]
\centering
\includegraphics[width=90mm]{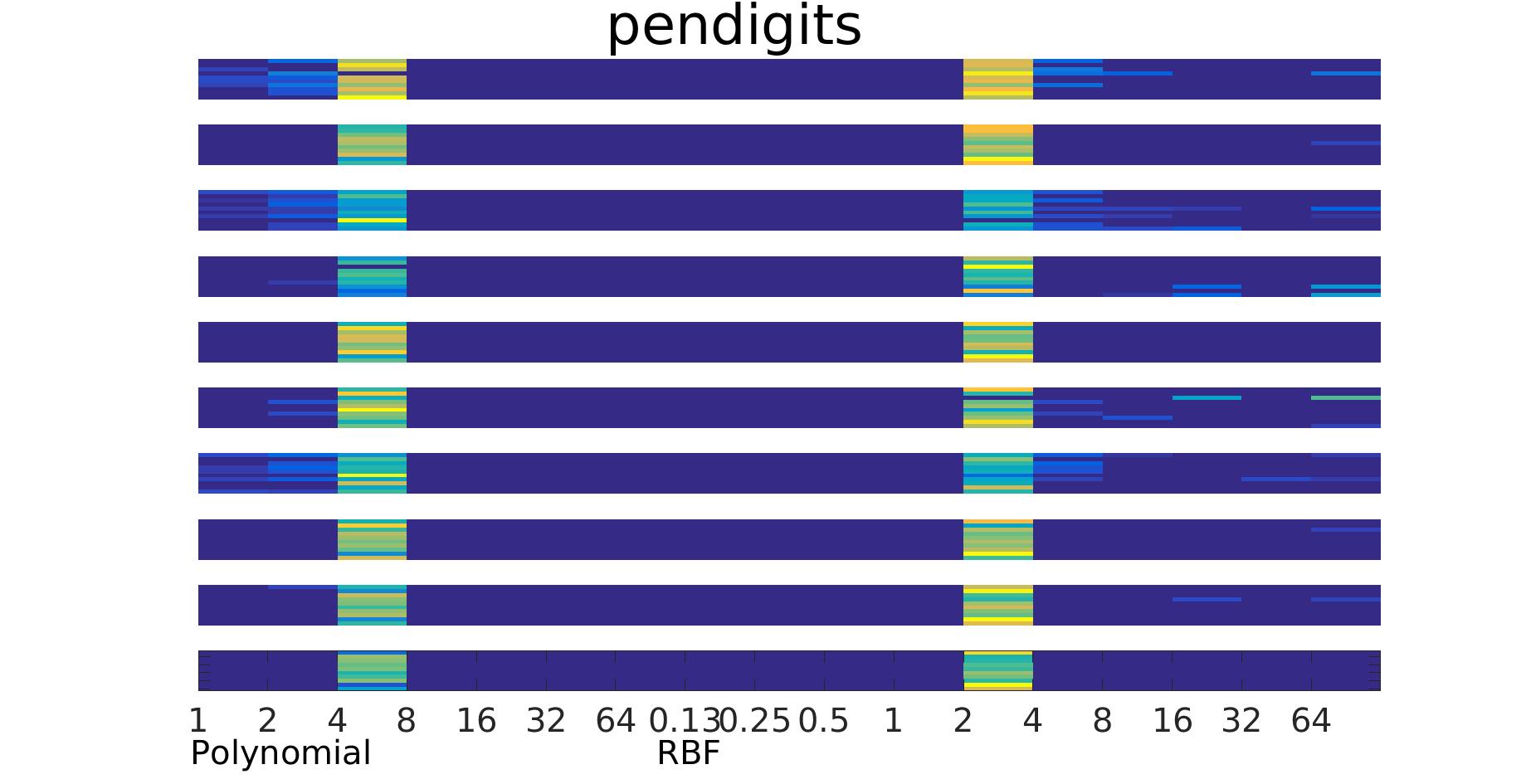}
\caption{Similar to Fig.~\ref{fig:example_nu_banana}, this figure shows an example of the estimated $\bss{\nu}$ on the dataset pendigits using the polynomial kernels with $d\in\{2^0,2^1,\cdots,2^6\}$ and the RBF kernels with $\sigma\in\{2^{-3},2^{-2},\cdots,2^6\}$.}
\label{fig:example_nu_pendigits}
\end{figure}
\begin{figure}[ht!]
\centering
\includegraphics[width=90mm]{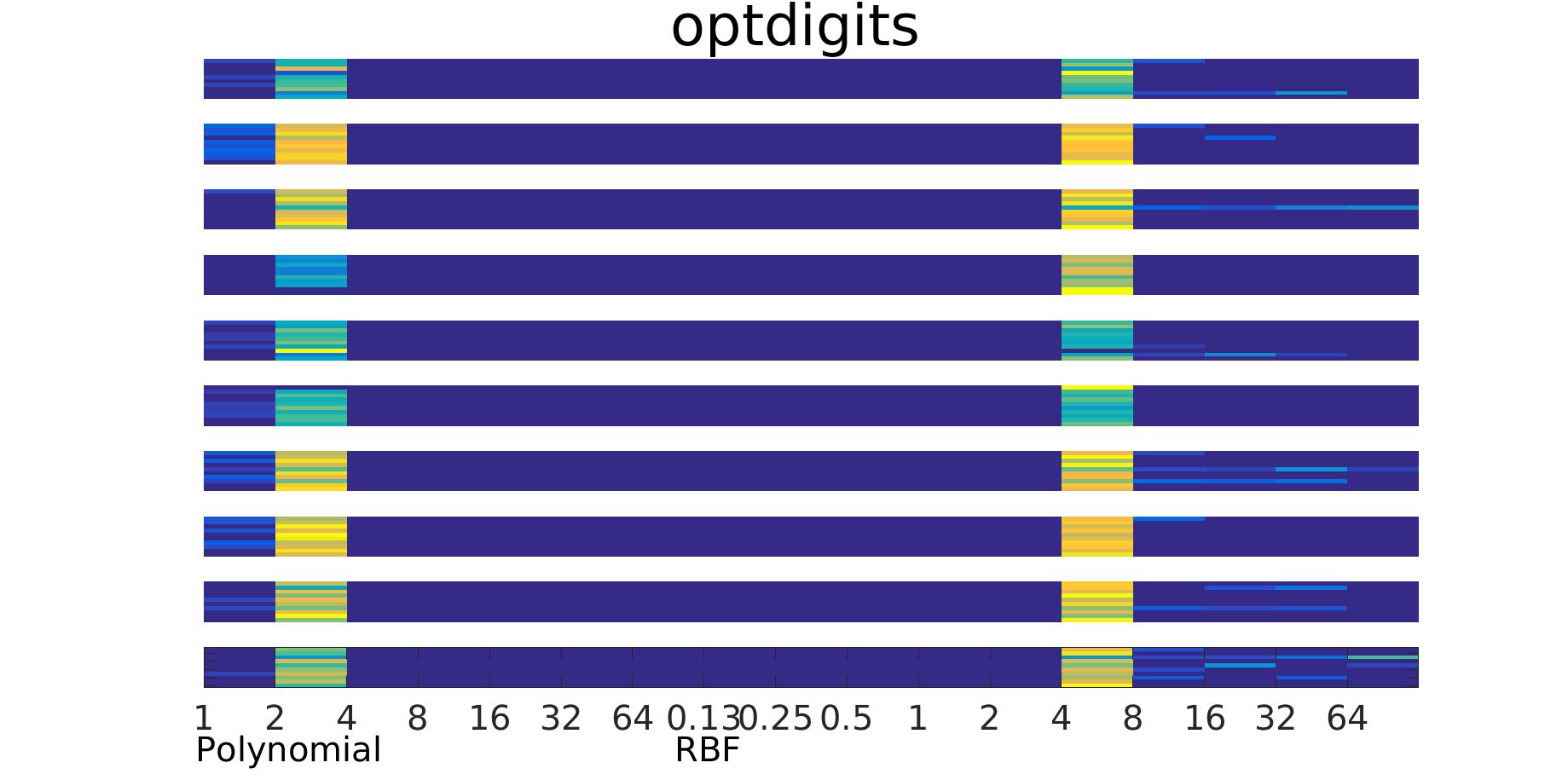}
\caption{This figure shows an example of the estimated $\bss{\nu}$ on the dataset optdigits. It shares the same setup as Fig.~\ref{fig:example_nu_pendigits}.}
\label{fig:example_nu_optdigits}
\end{figure}
\begin{figure}[ht!]
\centering
\includegraphics[width=90mm]{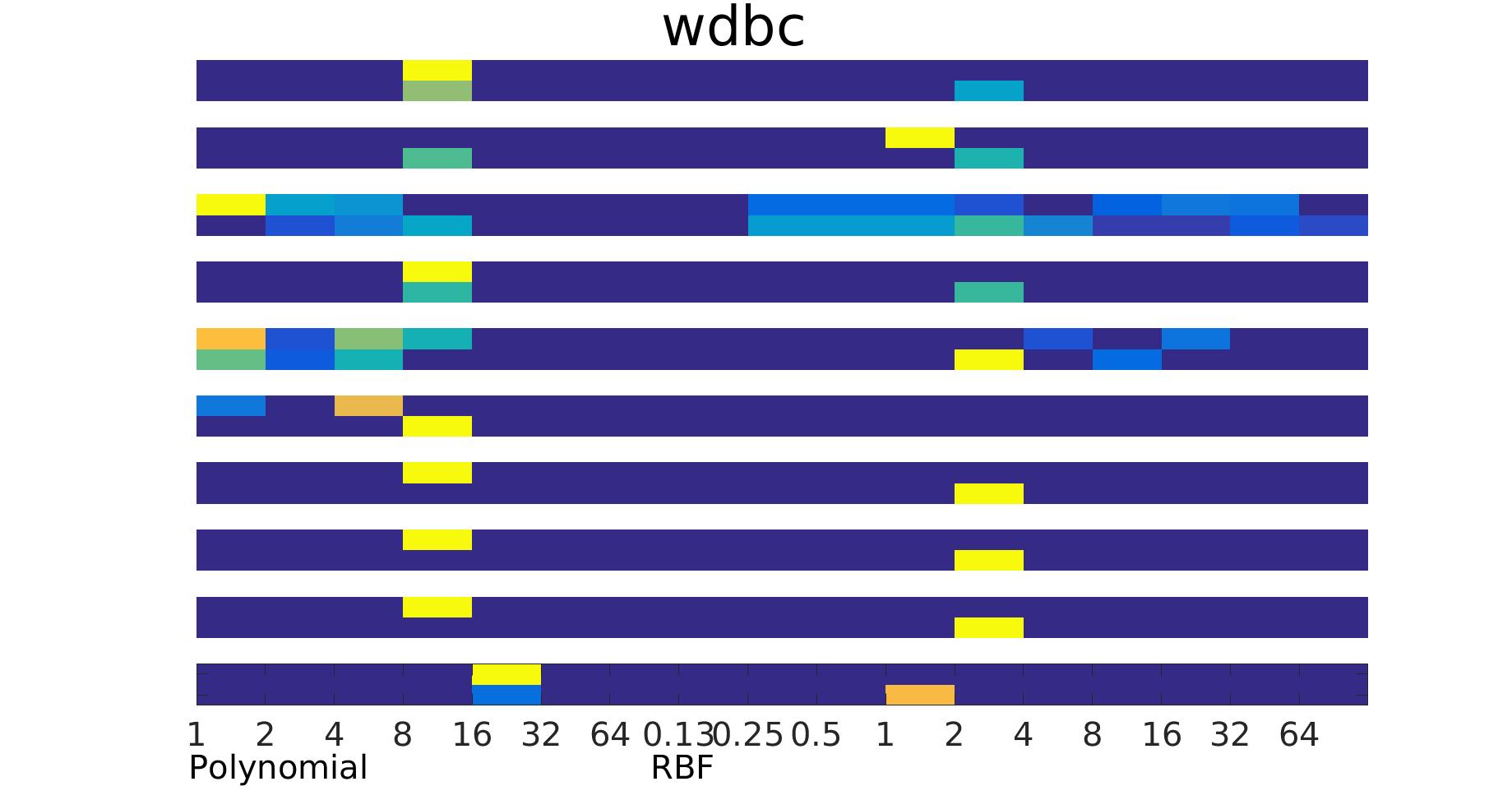}
\caption{This figure shows an example of the estimated $\bss{\nu}$ on the dataset wdbc. It shares the same setup as Fig.~\ref{fig:example_nu_pendigits}.}
\label{fig:example_nu_wdbc}
\end{figure}

\section{Conclusion}
\label{sec:conclusion}
In this paper, an automatic model selection technique has been presented for kernel classification methods using class-specific multiple kernel functions.
We motivate the proposal from a metric learning viewpoint, where the goal is to find a metric space such that the within-class distance is smaller than the between-class distance from a statistical point of view. Essentially, the selection is based on the underlying subspace model in the kernel-induced feature space. By evaluating an upper bound of the objective probability, we can select the best kernel function with respect to the lowest upper bound for each class using the CLAss-specific Subspace Kernel (CLASMK). Moreover, to further enhance the flexibility of the learning model, we introduce the CLAss-Specific Multiple-Kernel (CLASMK) model and a metric learning technique called CLASMK-Metric Learning for identifying the weighting coefficient of each feature vector induced by the corresponding kernel function.
A hierarchical learning structure is also proposed to improve the classification performance for a given base classifier by feature augmentation.
Empirical tests have shown promising results on various datasets.
As a future direction, tests using more types of kernel functions are under progress.
Moreover, feature pruning strategies are needed at each layer for large scale datasets. We are also investigating the possibilities of integrating the feature augmentation technique into a deep kernel network.

\section*{Acknowledgement}
This work has in part been funded by the Swedish Research Council (Vetenskapsr{\aa}det) under the contract number A0462701 which is gratefully acknowledged.

\appendix
\section{Lemma~\ref{lemma:lemma2}}
\begin{Lemma}
\label{lemma:lemma2}
Given the class-specific kernel model, and random vectors $\bss{\varphi}_c$, $\bss{\nu}_{\tilde{c}}$ from class $c$ and $\tilde{c}\neq c$, respectively. For any $\lambda>0$, such that
\begin{equation}
\label{eqa:condition2}
{\frac{\frac{1}{C-1}\left(\sum_{i\neq c}\mathbb{E}\left(\|\bs{U}_{i}^T\bss{\varphi}_{(c,i)}\|_2^2\right)+\sum_{j\neq \tilde{c}}\mathbb{E}\left(\|\bs{U}_{j}^T\bss{\nu}_{(\tilde{c},j)}\|_2^2\right)\right)}{\mathbb{E}\left(\|\bs{U}_{c}^T\bss{\varphi}_{(c,c)}\|_2^2\right)+\mathbb{E}\left(\|\bs{U}_{\tilde{c}}^T\bss{\nu}_{(\tilde{c},\tilde{c})}\|_2^2\right)}\leq \lambda}
\end{equation}
then $\mathbb{E}\left(D_{c,\tilde{c}}\right)\geq 2C\left( 1 - \frac{(C\lambda - \lambda +1 )(1-\sigma_e^2)}{C}\right)$.
\end{Lemma}
\begin{proof}
Let $\bss{\varphi}_c = \bs{U}\bss{\beta}_c+\bs{e}_c$ and $\bss{\nu}_{\tilde{c}} = \bs{U}\bss{\eta}_{\tilde{c}}+\bs{e}_{\tilde{c}}$ defined in Eq.~\eqref{eqa:model_decomp2}. We can then compute the between-class distance as follows:
\begin{eqnarray}
\nonumber
\mathbb{E}\left(D_{c,\tilde{c}}\right)&=&2\left(C - \sum_{i=1}^C \mathbb{E}\left(\bss{\beta}_{(c,i)}^T\bss{\eta}_{(\tilde{c},i)}\right)\right)\\
\nonumber
&\geq & 2\left(C - \sum_{i=1}^C \mathbb{E}\left(\|\bss{\beta}_{(c,i)}\|_2\|\bss{\eta}_{(\tilde{c},i)}\|_2\right)\right)\\
\label{eqa:proof_lemma2_0}
&\geq & 2\left(C - \frac{1}{2}\sum_{i=1}^C \mathbb{E}\left(\|\bss{\beta}_{(c,i)}\|_2 ^2+\|\bss{\eta}_{(\tilde{c},i)}\|_2^2\right) \right)
\end{eqnarray}
Let $\small{s = \sum_{i\neq c}\mathbb{E}\left(\|\bs{U}_{i}^T\bss{\varphi}_{(c,i)}\|_2^2\right)+\sum_{j\neq \tilde{c}}\mathbb{E}\left(\|\bs{U}_{j}^T\bss{\nu}_{(\tilde{c},j)}\|_2^2\right)}$. From Eq.~\eqref{eqa:condition2} and $\|\bss{\beta}_{(c,c)}\|_2^2=\|\bss{\eta}_{(\tilde{c},\tilde{c})}\|_2^2=1-\sigma_e^2$, we know that
\begin{eqnarray*}
s\leq 2\lambda(C-1)(1-\sigma_e^2)
\end{eqnarray*}
Therefore, we have:
\begin{eqnarray*}
\label{eqa:proof_lemma2_1}
\mathbb{E}\left(D_{c,\tilde{c}}\right)&\geq& 2\left(C - \frac{1}{2}\left(\|\bss{\beta}_{(c,c)}\|_2^2+\|\bss{\eta}_{(\tilde{c},\tilde{c})}\|_2^2 + s\right)\right)\\
&\geq &2\left(C - \frac{1}{2}\left(2\left(1-\sigma_e^2\right) + 2\lambda(C-1)(1-\sigma_e^2)\right)\right)\\
& = & 2C\left( 1 - \frac{(C\lambda - \lambda +1 )(1-\sigma_e^2)}{C}\right)
\end{eqnarray*}
\end{proof}

\section{Lemma~\ref{lemma:d_multiclass}}
\label{app:d_multiclass}
In this Lemma, we discuss the between-class distance for all classes in a one-against-one fashion. Generalization to unbalanced label problems can be readily derived using the mechanism in Theorem.~\ref{thm:theorem1}.
\begin{Lemma}
\label{lemma:d_multiclass}
Given the class-specific model, assume that $p_1=\cdots=p_C=\frac{1}{C}$. If $\exists\lambda<1$, such that
\begin{equation}
\label{eqa:condition3}
{\frac{\frac{1}{C-1}\sum_{\forall c}\sum_{i\neq c}\mathbb{E}\left(\|\bs{U}_{i}^T\bss{\varphi}_{(c,i)}\|_2^2\right)}{\sum_{\forall c}\mathbb{E}\left(\|\bs{U}_{c}^T\bss{\varphi}_{(c,c)}\|_2^2\right)}\leq \lambda}
\end{equation}
then
\begin{equation}
\label{eqa:lemma3}
\mathbb{E}\left(D_{b}\right)\geq 2C\left( 1 - \frac{(C\lambda - \lambda +1 )(1-\sigma_e^2)}{C}\right)
\end{equation}
\end{Lemma}
\begin{proof}
From Eq.~\eqref{eqa:proof_lemma2_0} and Eq.~\eqref{eqa:proof_lemma3_0}, we know that:
\begin{eqnarray*}
\nonumber
{ \mathbb{E}\left(D_{b}\right) }&\geq& {\frac{1}{C(C-1)}\sum_{\forall c}\sum_{\forall \tilde{c}\neq c} 2\left(C - \frac{1}{2}\sum_{i=1}^C \mathbb{E}\left(\|\bss{\beta}_{(c,i)}\|_2 ^2+\|\bss{\eta}_{(\tilde{c},i)}\|_2^2\right) \right)}\\
\nonumber
&=&{2\left(C -\frac{1}{2C(C-1)}\sum_{\forall c}\sum_{\forall \tilde{c}\neq c} \sum_{i=1}^C \mathbb{E}\left(\|\bss{\beta}_{(c,i)}\|_2 ^2+\|\bss{\eta}_{(\tilde{c},i)}\|_2^2\right) \right)}
\end{eqnarray*}
Furthermore, since
\begin{eqnarray*}
&&{\sum_{\forall c}\sum_{\forall \tilde{c}\neq c} \sum_{i=1}^C \mathbb{E}\left(\|\bss{\beta}_{(c,i)}\|_2 ^2+\|\bss{\eta}_{(\tilde{c},i)}\|_2^2\right)}\\
 &=& {2C(C-1)\left((1-\sigma_e^2)+\sum_{\forall c}\sum_{\forall i\neq c}\mathbb{E}\left(\|\bss{\beta}_{(c,i)}\|_2 ^2 \right)\right)},
\end{eqnarray*}
we have
\begin{eqnarray}
\label{eqa:proof_lemma3_1}
 \mathbb{E}\left(D_{b}\right) \geq  {2\left(C -\left((1-\sigma_e^2)+\sum_{\forall c}\sum_{\forall i\neq c}\mathbb{E}\left(\|\bss{\beta}_{(c,i)}\|_2 ^2\right) \right)\right)}.
\end{eqnarray}
Therefore, if Eq.~\eqref{eqa:condition3} holds, i.e.
\begin{eqnarray*}
\label{eqa:proof_lemma3_2}
\frac{1}{C-1}\sum_{\forall c}\sum_{\forall i\neq c}\mathbb{E}\left(\|\bss{\beta}_{(c,i)}\|_2 ^2\right)&\leq& \lambda C \sum_{\forall c}\mathbb{E}\left(\|\bss{\beta}_{(c,c)}\|_2 ^2\right)\\
& = & \lambda C(1-\sigma_e^2)
\end{eqnarray*}
then
$\eqref{eqa:proof_lemma3_1}\geq 2\left(C - \left((1-\sigma_e^2) + \lambda C(C-1)(1-\sigma_e^2)\right) \right)$.
\end{proof}

\section{Proof of Theorem~\ref{thm:theorem2}}
\begin{proof}
The proof can be illustrated using same routine of the proof for Theorem~\ref{lemma:theorem1} together with Lemma~\ref{lemma:d_multiclass}.
\end{proof}

\end{document}